\documentclass{article} 
\usepackage{iclr2023_conference,times}

\usepackage{amsmath,amsfonts,bm}

\def\eqref#1{equation~\ref{#1}}

\def\1{\bm{1}}

\DeclareMathAlphabet{\mathsfit}{\encodingdefault}{\sfdefault}{m}{sl}
\SetMathAlphabet{\mathsfit}{bold}{\encodingdefault}{\sfdefault}{bx}{n}

\usepackage{hyperref}
\usepackage{url}
\usepackage{microtype}
\usepackage{graphicx}
\usepackage{booktabs} 
\usepackage{listings}
\usepackage{hyperref}
\usepackage{subcaption}
\usepackage{sidecap}
\usepackage{xcolor} 
\usepackage{ifthen}
\usepackage{adjustbox}

\newboolean{include-notes}
\setboolean{include-notes}{true}

\newcommand{\nb}[3]{\ifthenelse{\boolean{include-notes}}{{\colorbox{#2}{\bfseries\sffamily\scriptsize\textcolor{white}{#1}}}{\ \textcolor{#2}{\sf\small\textit{#3}}}}{}}

\newcommand{\colorred}{\color{black}}

\iclrfinalcopy

\usepackage{amsmath}
\usepackage{amssymb}
\usepackage{mathtools}
\usepackage{amsthm}
\usepackage{multirow}

\usepackage[capitalize,noabbrev]{cleveref}

\usepackage{comment}
\usepackage[utf8]{inputenc} 
\usepackage[T1]{fontenc}    
\usepackage{url}            
\usepackage{booktabs}       
\usepackage{amsfonts}       
\usepackage{nicefrac}       
\usepackage{microtype}      
\usepackage{appendix}
\usepackage{microtype}      
\usepackage{xcolor}         

\usepackage{geometry}
\usepackage{svg}
\usepackage{amsmath}
\usepackage{amsthm}
\usepackage{mathrsfs}
\usepackage{algorithm}
\usepackage{algorithmicx}
\usepackage{algpseudocode}
\usepackage{amssymb}
\usepackage[normalem]{ulem}
\usepackage{enumitem}

\theoremstyle{plain}
\newtheorem{theorem}{Theorem}

\theoremstyle{definition}

\theoremstyle{remark}

\usepackage[textsize=tiny]{todonotes}

\title{CASA: Bridging the Gap between Policy Improvement and Policy Evaluation with Conflict Averse Policy Iteration}

\author{Changnan Xiao \\
\tiny xiaochangnan@bytedance.com \\
\And
Haosen Shi \\
\tiny shihaosen98@gmail.com \\
\And
Jiajun Fan \\
\tiny fanjj21@mails.tsinghua.edu.cn \\
\And
Shihong Deng \\
\tiny dengshihong@bytedance.com \\
\And
Haiyan Yin \\
\tiny yinhaiyan@outlook.com \\
}

\begin{document}

\maketitle

\begin{abstract}

We study the problem of model-free reinforcement learning, which is often solved following the principle of Generalized Policy Iteration (GPI). 
While GPI is typically an interplay between policy evaluation and policy improvement, most conventional model-free methods with function approximation assume the independence of 
GPI steps, despite of the inherent connections between them. 
In this paper, we present a method that attempts to {eliminate} the inconsistency between policy evaluation step and policy improvement step, leading to a conflict averse GPI solution with gradient-based functional approximation.  
Our method is capital to balancing exploitation and exploration between policy-based and value-based methods and is applicable to existing policy-based and value-based methods. 
We conduct extensive experiments to study theoretical properties of our method and demonstrate the effectiveness of our method on Atari 200M benchmark.
\end{abstract}

\section{INTRODUCTION}
\label{sec:intro}

Model-free reinforcement learning has made many impressive breakthroughs in a wide range of Markov Decision Processes (MDP) ~\citep{alpha_star,ftw,agent57}.
Overall, the methods could be cast into two categories, value-based methods such as DQN~\citep{dqn} and Rainbow~\citep{rainbow}, and  policy-based methods such as TRPO~\citep{trpo}, PPO~\citep{ppo} and IMPALA~\citep{impala}.  


Value-based methods learn state-action values and select the action according to their values. 
The main target of value-based methods is to approximate the fixed point of the Bellman equation through the generalized policy iteration (GPI) ~\citep{sutton}, which generally consists of policy evaluation and policy improvement. 
One characteristic of the value-based methods is that unless a more accurate state-action value is estimated by iterations of the policy evaluation, the policy will not be improved. 
Previous works equip value-based methods with many carefully designed structures to achieve more promising reward learning and sample efficiency ~\citep{dueling_q,priority_q,r2d2}.

Policy-based methods learn a parameterized policy directly without consulting state-action values.
One characteristic of policy-based methods is that they incorporate a policy improvement phase in every training step, while  
in contrast,  the value-based methods only change the policy after the action corresponding to the highest state-action values is changed. 
In principle, policy-based methods perform policy improvement more frequently than value-based methods.



We notice that value-based and policy-based methods locate at the two extremes of GPI, where value-based methods won't improve the policy until a more accurate policy evaluation is achieved, while policy-based methods improve the policy for every training step even when the policy evaluation hasn't converged. To mitigate the defect of each, we pursuit a technique that is capable of balancing  between the two extremes flexibly. 
We first study the gradients between policy improvement and policy evaluation and notice that they show a positive correlation statistically during the entire training process. To find out if there exists a way that the gradients of the policy improvement and the policy evaluation are parallel, we propose CASA, \textbf{C}ritic \textbf{AS} an \textbf{A}ctor, {\colorred which satisfies a weaker compatible condition \citep{sutton1999policy} and enhances gradient consistency between policy improvement and policy evaluation.}

With further delving into the properties of CASA, we find CASA is an innovative combination of value-based and policy-based methods. 
When the policy-based methods are equipped with CASA, the collapse to the sub-optimal solution as the entropy goes to zero is prevented by the evaluation of the state-action values, which encourages exploration.  
When the value-based methods are equipped with CASA, the policy improvement via policy gradient is equivalent to the evaluation of the state-action values and a self-bootstrapped policy improvement, which enhances exploitation.


To enable CASA for a large scale off-policy learning, we introduce Doubly-Robust Trace (DR-Trace), which
exploits doubly-robust estimator ~\citep{dr} and guarantees the synchronous convergence of the state-action values and the state values. 


Our main contributions are as follows:
\begin{enumerate}[label=(\roman*),leftmargin=*]
\item We present a novel method CASA which {\colorred enhances gradient consistency} between policy evaluation and policy improvement and present extensive studies on the behavior of the gradients.
\item We demonstrate CASA could be freely applied to both policy-based and value-based algorithms with motivating examples.
\item We present extensive empirical study on Atari benchmark , where our conflict averse algorithm brings substantial improvements over the baseline methods. 
\end{enumerate}

\section{Preliminary}
\label{sec:bg}

Consider an infinite-horizon MDP, defined by a tuple $(\mathcal{S}, \mathcal{A}, p, r, \gamma)$, where $\mathcal{S}$ is the state space, $\mathcal{A}$ is the action space, $p: \mathcal{S} \times \mathcal{A} \times \mathcal{S} \to [0, 1]$ is the state transition probability function, $r: \mathcal{S} \times \mathcal{A} \to \mathbb{R}$ is the reward function, and $\gamma$ is the discounted factor.
The policy is a mapping  $\pi: \mathcal{S} \times \mathcal{A} \to [0, 1]$ which assigns a distribution over the action space given a state.

The objective of reinforcement learning is to maximize the \emph{return}, or cumulative discounted rewards,  {\colorred 
\begin{equation}
\mbox{maximize} \ \mathcal{J} = \mathbb{E}_{traj \sim \pi} 
\left[ \sum_t \gamma^t r(s_t, a_t)\right],
\end{equation}
where $traj = \{s_0, a_0, r_0, \dots\}$ is a trajectory sampled by $\pi$ with policy-environment interaction.}

Value-based methods maximize $\mathcal{J}$ by estimating various type of value functions: the state value function is defined as {\colorred $V^{\pi}(s) = \mathbb{E}_{\pi}\left[ \sum_t \gamma^t r_{t} | s_0=s \right]$}, the state-action value function is defined as {\colorred $Q^{\pi}(s, a) = \mathbb{E}_{\pi} \left[ \sum_t \gamma^t r_{t} | s_0=s, a_0=a \right]$}; the advantage function is defined as $A^{\pi}(s, a) = Q^{\pi}(s, a) - V^{\pi}(s)$.
The objective of maximizing the value functions in value-based methods 
can be improved through GPI until converging to the optimal policy.
For the approximated state-value function $Q_\theta$ that estimates $Q^\pi$, the policy evaluation is conducted by:{\colorred 
\begin{equation}
\mbox{minimize}\ \mathbb{E}_{\pi}
 [(Q^\pi(s, a) - Q_\theta (s, a))^2],
 \end{equation}}
where $Q^\pi$ is estimated by various methods, e.g., $\lambda$-return~\citep{Sutton88lambda} and ReTrace~\citep{retrace}.
The policy improvement is usually achieved by greedily selecting actions with the highest state-action values.

Policy-based methods maximize $\mathcal{J}$ by optimizing some parameterized policy $\pi_\theta$ according to the policy gradient theorem \citep{sutton}, 
\begin{equation}
\nabla_\theta \mathcal{J} = 
\mathbb{E}_{\pi} 
[\Psi(s, a) \nabla_\theta \log \pi_\theta(a|s)].
\end{equation}
{\colorred The vanilla policy gradient uses $\Psi = \sum_{t=0}^\infty \gamma^t r_{t}$.}
Actor-critic algorithms approximate $\Psi(s,a)$ in the form of baseline, e.g., 
IMPALA~\citep{impala} adopts $\Psi(s, a) = r + \gamma V^{\Tilde{\pi}} (s') - V_{\theta} (s)$ and uses V-Trace to estimate $V^{\Tilde{\pi}}$.


\section{Methodology}
\label{sec:casa}

\subsection{Motivation}
\label{sec:motivation}
We use $V_{\theta}$ to estimate $V^\pi$, $Q_{\theta}$ to estimate $Q^\pi$ and $\pi_{\theta}$ to represent the policy,
where $\theta$ represents all parameters to be optimized.
{\colorred In this work, there is one backbone and two individual heads after the backbone. 
The advantage function and the policy share one head, and state value function is the other head. 
Hence the policy reuses all parameters of value functions except that temperature $\tau$ is only for the policy.
We keep $\tau$ static in this work.}
We use \textbf{E} to represent the policy evaluation, which gives the ascent direction of the gradient by $\theta \leftarrow \theta + \eta \mathbb{E}_{\pi} [(Q^{\pi} - Q_{\theta}) \nabla_{\theta} Q_{\theta}]$. 
We use \textbf{I} to represent the policy improvement, which gives
$\theta \leftarrow \theta + \eta \mathbb{E}_\pi [(Q^\pi-V_\theta) \nabla_\theta \log \pi_\theta]$.

\begin{figure}[t]
\centering 
\vspace{-0.5cm}
\begin{minipage}[t]{0.38\linewidth}
\centering
\includegraphics[width=0.95\linewidth]{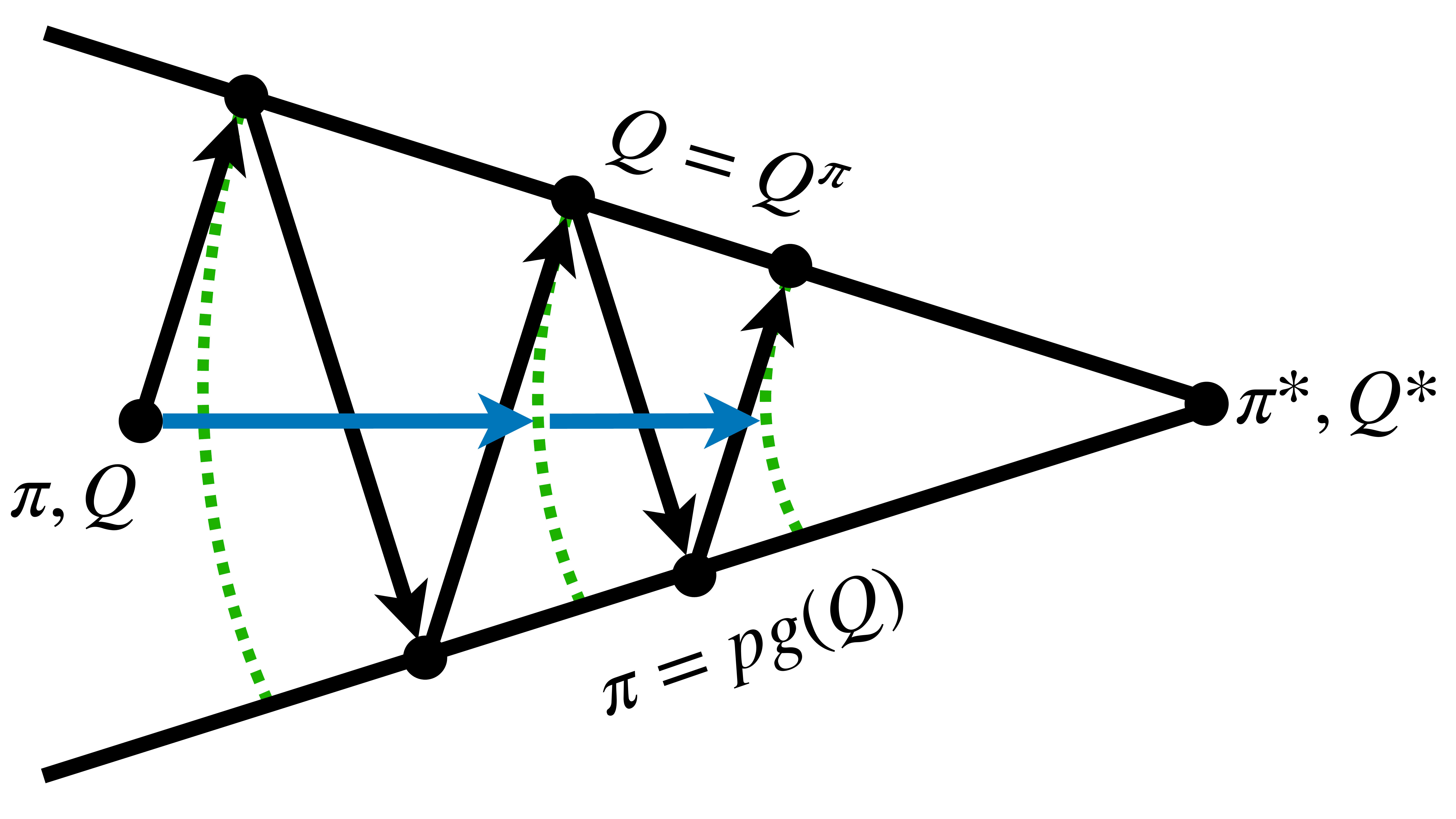}
\caption{ \small 
The GPI process in our work. 
Unlike ~\citep{sutton}, we evaluate $\pi$ by $Q$ instead of $V$, and we improve $\pi$ using policy gradient ascent ($pg$ for brevity) instead of greedy. The learning procedure is shown by the black arrows, i.e., $\textbf{E}\rightarrow\textbf{I} \rightarrow \textbf{E} \rightarrow \textbf{I} \cdots$. 
} 
\label{fig:gpi}
\vskip -0.2in
\end{minipage}
\hfill
\begin{minipage}[t]{0.58\linewidth}
\centering
\includegraphics[width=0.65\linewidth]{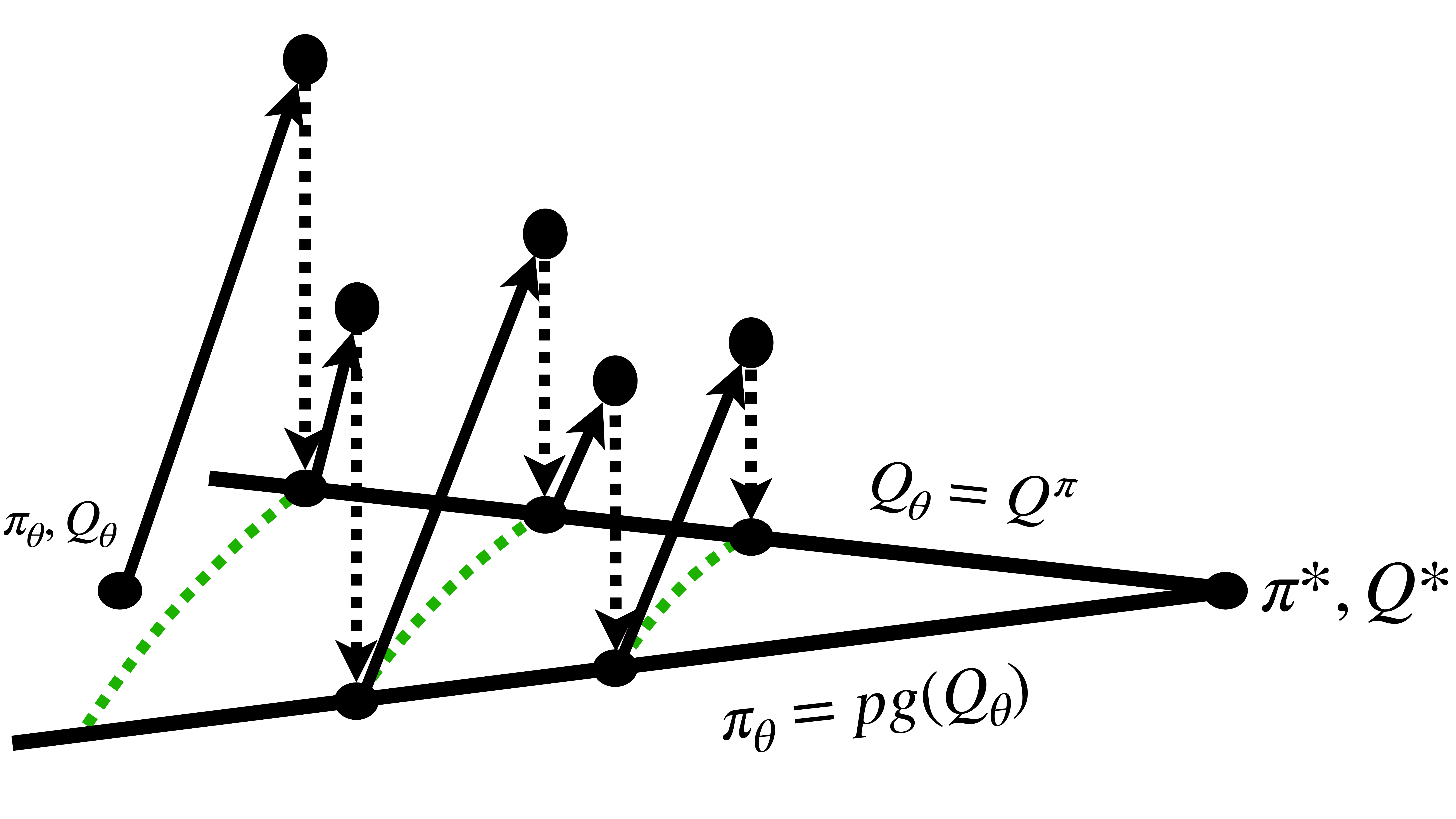}
\caption{\small 
GPI with function approximation. 
Due to the constraint of approximated function space, the ideal policy iteration cannot be actually achieved. The underlying process of GPI with function approximation can be regarded as doing policy improvement and policy evaluation in an ideal space then being projected back into the approximated function space ~\citep{sutton, op_reinforce}.
 }
\label{fig:gpi2}
\end{minipage}
\vskip -0.1in
\end{figure}

Let's recap the GPI process as shown in Figure \ref{fig:gpi}.
To get rid of the function approximation error, we first assume the approximation function enjoys infinite capacity. 
We use $<x, y>$ to denote the angle between two vectors, where $<x, y> = \arccos (\frac{x\cdot y}{||x||\cdot||y||})$ {\colorred with $\arccos: [-1, 1] \rightarrow [0, \pi]$}. 
We define an important notion $\beta$, which represents the angle between the gradient ascent directions of \textbf{I} and \textbf{E}, as follows,
{\colorred 
\begin{equation}
    \beta \overset{def}{=} <\mathbb{E}_\pi[(Q^\pi-Q_\theta)\nabla_\theta Q_\theta],\, \mathbb{E}_\pi[(Q^\pi-V_\theta) \nabla_\theta \log \pi_\theta]>.
\end{equation}
}
{\colorred When $\beta = 0$ i.e.$\cos(\beta) = 1$, \textbf{I} and \textbf{E} become parallel to each other,  which is the blue arrow in Figure \ref{fig:gpi},
and there is no conflict between the gradient ascent directions of $\textbf{I}$ and $\textbf{E}$ anymore.
When $\beta = \pi / 2$ i.e.$\cos(\beta) = 0$, \textbf{I} and \textbf{E} are perpendicular. 
When $\beta = \pi$ i.e.$\cos(\beta) = -1$, \textbf{I} and \textbf{E} are toward exactly opposite directions. }


Next, we assume the representation capacity of the approximation function is limited.
When the function approximation is involved, i.e. $Q^\pi$ is estimated by $Q_\theta$ and $\pi$ is approximated by $\pi_\theta$, from the view of operators ~\citep{op_reinforce}, each of \textbf{I} and \textbf{E} can be further decomposed into two operators, as shown in Figure \ref{fig:gpi2}.
One is to do the policy improvement and the policy evaluation, the other is to project into the restricted function space.
When $\beta > 0$, GPI with function approximation would involve two projection operators in each iteration, which introduces inevitable approximation error.
When $\beta = 0$, if the function approximation error is not considered, we find that the gradient conflict between \textbf{I} and \textbf{E} would be totally eliminated.
If we consider the limitation of the approximation function, similar to the blue arrow in Figure \ref{fig:gpi}, one iteration (represented by two black arrows and two dotted arrows) can be united into one arrow and one dotted arrow (not shown in Figure \ref{fig:gpi2} but analogy to the blue arrow in Figure \ref{fig:gpi}), where the gradient conflict is eliminated and the two projection operators are reduced to one correspondingly.

As stated above, if $\beta = 0$ holds, we can expect that the gradient conflict between the policy improvement and the policy evaluation is eliminated and the function approximation error could be reduced. 
{\colorred However, $\beta$ is usually estimated by sampling with stochasticity. 
It's difficult to let $\beta = 0$ by optimizing $\theta$. 
Instead, we consider another notion $\chi$ by removing step sizes and taking expectation outside, where the angle of each state is fully controllable by $\theta$.
\begin{equation}
    \chi \overset{def}{=} \mathbb{E}_\pi [\cos <\nabla_\theta Q_\theta, \nabla_\theta \log \pi_\theta>].
\end{equation}
In fact, $\chi$ is highly correlated to compatible value function \citep{sutton1999policy},
and Theorem \ref{thm:connect_cond} shows that $\chi = 1$ is the necessary condition for the compatible condition $\nabla_\theta Q_\theta = \nabla_\theta \log \pi_\theta$, which is a weaker compatible condition.
More details about compatible value function are in Appendix \ref{app:comp_v}.
}

\begin{figure}[t!]
	\centering
	\begin{minipage}[c]{0.75\textwidth}
\includegraphics[width=\linewidth]{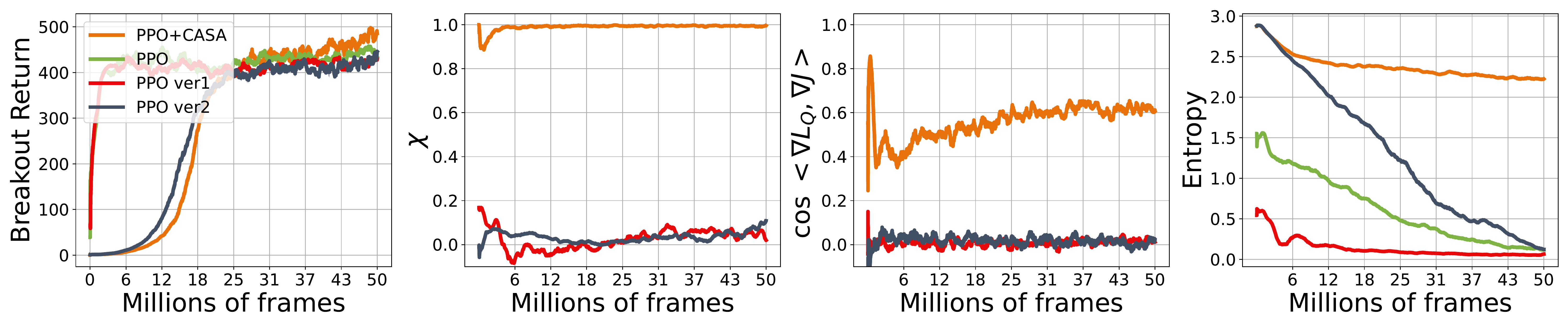}
\includegraphics[width=\linewidth]{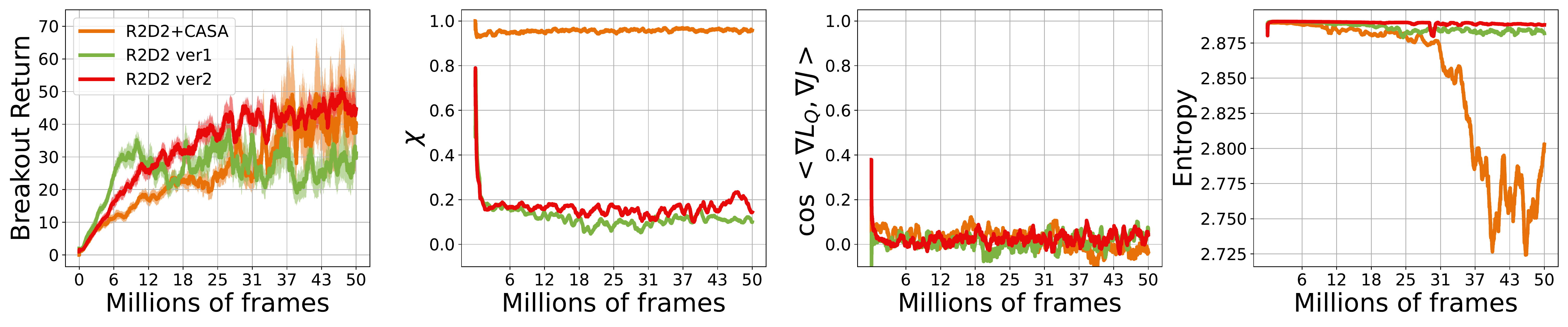}
  \end{minipage} \hfill
  \begin{minipage}[c]{0.23\textwidth}
    \caption{
    \small
    \emph{Return}, $\chi$, $\cos(\beta)$ and \emph{entropy}. PPO is adjusted with two additional versions to evaluate state-action values. R2D2 uses a surrogate policy to approximate policy gradient. Entropy of R2D2 is entropy of Boltzmann policy on state-action values. 
    Details are in Appendix \ref{app:mtv}.}
    \label{fig:mtv}
    \end{minipage}
\end{figure}

To further understand the behavior of $\beta$ {\colorred and $\chi$}, we track $\cos(\beta)$ {\colorred and $\chi$} of two algorithms PPO and R2D2  as representatives for policy-based and value-based methods, respectively. We show an important fact in Figure \ref{fig:mtv} that {\colorred both $\chi$ and $\cos(\beta)$ are statistically positive for both original version and adjusted versions,
which means that $\arccos(\chi)$ and $\beta$ are likely to be less then $\pi / 2$ with neural network approximated functions.}
The aforementioned conceptual and empirical findings  inspire us to raise the following question on GPI:
{\colorred whether we can guarantee $\chi = 1$, so that $\cos(\beta)$ is also closer to $1$.}


\subsection{Formulation}
\label{sec:formula}

Denote $\tau \in \mathbb{R}_+$ to be a positive temperature and $sg$ to be a $stop\ gradient$ operator. 
CASA can estimate $V_\theta$ and $A_\theta$ by any function parameterized by $\theta$, where $\pi_\theta$ and $Q_\theta$ are derived as follows:
\begin{equation}
\label{eq:casa}
\left\{
    \begin{aligned}
        &\pi_\theta(\cdot|s) = \text{softmax}(A_\theta(s, \cdot) /
        \tau), \\
        &\Bar{A}_\theta(s, a) = A_\theta(s, a) - \sum_{a'} sg(\pi_\theta(a'|s)) A_\theta(s, a'), \\
        &Q_\theta(s, a) = \Bar{A}_\theta(s, a) + sg(V_\theta(s)).
    \end{aligned}
\right. 
\end{equation}

Note that there exist two $sg$ operators in \eqref{eq:casa}.
The first $sg$ operator is used for computing advantage as $\Bar{A}_\theta = A_\theta - \mathbb{E}_\pi [A_\theta] = A_\theta - sg(\pi_\theta) \cdot A_\theta$,
where the $sg$ operator here guarantees the gradients between policy improvement and policy evaluation are parallel, which we elaborate later.
Intuitively, this $sg$ operator also means that we keep $\pi_\theta$ unchanged when evaluating the policy $\pi_\theta$.
The second $sg$ operator exists in $Q_\theta = \Bar{A}_\theta + sg(V_\theta)$.
As ~\citep{simsiam} regards $sg$ in siamese representation learning as a case of EM-algorithm ~\citep{em}, a similar interpretation exists here.
$Q_\theta = \Bar{A}_\theta + sg(V_\theta)$ decomposes the estimation of $Q_\theta$ into a two stage problem, where the first is to estimate the advantage of each action without changing the expectation, the second is to estimate the expectation.

The \eqref{eq:casa} includes a straightforward refinement of dueling-DQN.
{\colorred We know dueling-DQN estimates $Q^\pi$ by $Q_\theta =  A_\theta + V_\theta$, but it cannot guarantee $\mathbb{E}_\pi[A_\theta] = 0$ i.e. $\mathbb{E}_\pi [Q_\theta] = V_\theta$ due to the function approximation error. 
But if we estimate $Q_\pi$ by $Q_\theta =  A_\theta - \mathbb{E}_\pi[A_\theta] + V_\theta$, it satisfies the necessary condition $\mathbb{E}_\pi [Q_\theta] = V_\theta$ without loss of generality.}

\subsection{Path Consistency Between Policy Evaluation and Policy Improvement}
\label{sec:equiv}

For brevity, we omit $\theta$ and $V, Q, A, \pi$ are all approximated functions.
Denote the estimations of $V$ and $Q$ as $V^{\Tilde{\pi}}$ and $Q^{\Tilde{\pi}}$ respectively. For instance, one choice is to calculate $V^{\Tilde{\pi}}$ and $Q^{\Tilde{\pi}}$ by V-Trace ~\citep{impala} and ReTrace ~\citep{retrace} respectively.

At training time, the policy evaluation is achieved by updating $\theta$ to minimize,
$$
\begin{aligned}
    L_V(\theta) = \mathbb{E}_\pi [ (V^{\Tilde{\pi}} - V)^2 ], \ 
    L_Q(\theta) = \mathbb{E}_\pi [ (Q^{\Tilde{\pi}} - Q)^2 ],
\end{aligned}
$$ 
which gives the ascent direction of $\theta$ by:
\begin{equation}
\label{eq:grad_qv}
    \begin{aligned}
        \nabla_\theta L_V(\theta)
        = \mathbb{E}_\pi \left[ (V^{\Tilde{\pi}} - V) \nabla_\theta V \right], \ 
        \nabla_\theta L_Q(\theta)
        = \mathbb{E}_\pi \left[ (Q^{\Tilde{\pi}} - Q) \nabla_\theta Q \right].
    \end{aligned}
\end{equation}
And we make the policy improvement by policy gradient, which gives the ascent direction of $\theta$ by: 
\begin{equation}
\label{eq:grad_pi}
\begin{aligned}
    \nabla_\theta \mathcal{J}(\tau, \theta) = \mathbb{E}_\pi \left[\tau (Q^{\Tilde{\pi}}  - V ) \nabla_\theta \log \pi \right],
\end{aligned}
\end{equation}
where $\mathcal{J} (\tau, \theta) = \tau \mathbb{E}_\pi [\sum \gamma^t r_t]$.
It takes an additional $\tau$, which frees the scale of gradient from $\tau$.

The final gradient ascent direction of $\theta$ is given by:
\begin{equation}
    \label{eq:grad_all}
    \alpha_1 \nabla_\theta L_V + \alpha_2 \nabla_\theta L_Q + \alpha_3 \nabla_\theta \mathcal{J}.
\end{equation}

With $(V, Q, \pi)$ defined in \eqref{eq:casa}, 
by Lemma \ref{lemma_app:vannila_grad}, we have, 
\begin{equation}
    \label{eq:vannila_grad}
\nabla_\theta Q =  (\textbf{1} - \pi) \nabla_\theta A = \tau \nabla_\theta \log \pi .
\end{equation}
For brevity, denote the shared gradient path as
$\textbf{g} = (\textbf{1} - \pi) \nabla_\theta A.$

Plugging \eqref{eq:vannila_grad} into \eqref{eq:grad_qv} \eqref{eq:grad_pi}, we have,
\begin{equation}
\label{eq:grad_qv_simple}
\begin{aligned}
        \nabla_\theta L_Q = \mathbb{E}_\pi \left[ (Q^{\Tilde{\pi}} - Q) \textbf{g} \right], 
\nabla_\theta \mathcal{J} = \mathbb{E}_\pi \left[ (Q^{\Tilde{\pi}} - V) \textbf{g} \right].
\end{aligned}
\end{equation}
{\colorred By \eqref{eq:grad_qv_simple}, $\nabla_\theta L_Q$ and $\nabla_\theta \mathcal{J}$ walk along the same vector direction of gradient path $\textbf{g}$ for each state.}
{\colorred By \eqref{eq:vannila_grad}, this is exactly the case $\chi = 1$.}
Since all parameters to estimate $Q$ and $\pi$ are shared except for $\tau$, we call it \textbf{C}ritic \textbf{AS} an \textbf{A}ctor.

If we make a subtraction between $\nabla_\theta L_Q$ and $\nabla_\theta \mathcal{J}$, we have,
\begin{equation}
\label{eq:grad_relation}
    \nabla_\theta \mathcal{J} = \nabla_\theta L_Q + \mathbb{E}_\pi \left[(Q - V) \textbf{g} \right].
\end{equation}
We know $\mathbb{E}_\pi \left[ (Q - V) \textbf{g} \right]$ is a self-bootstrapped policy gradient with function approximated $Q$.
Recalling the fact that the value-based methods improves the policy by greedily selecting actions according to $Q$, if we apply $\nabla_\theta \mathcal{J}$ on $\theta$, it additionally utilizes $Q$ to do policy improvement. 
This is a more greedy usage of $Q$ to improve policy than its usual usage. 


If we exploit the structural information as $(V, Q, \pi)$ defined by \eqref{eq:casa}, by Lemma \ref{lemma_app:eqiv_pg_ent},
$$
\mathbb{E}_\pi \left[(Q - V) \textbf{g} \right] 
= \tau \mathbb{E}_\pi \left[(Q - V) \nabla_\theta \log \pi \right]
= - \tau^2 \nabla_\theta \textbf{H}[\pi],
$$
then we have,
\begin{equation}
\label{eq:ent_reg}
    \nabla_\theta L_Q 
    = \nabla_\theta \mathcal{J} + \tau^2 \nabla_\theta \textbf{H}[\pi].
\end{equation}

The \eqref{eq:ent_reg} shows $\nabla_\theta L_Q$ is a policy gradient with an entropy regularization.
If we apply $\nabla_\theta L_Q$ on $\theta$ for policy-based methods, an entropy regularization works implicitly by $\alpha_2 \nabla_\theta L_Q$ in \eqref{eq:grad_all}, which prevents the policy collapse to a sub-optimal solution. 

\subsection{DR-Trace and Off-Policy Training}
\label{sec:dr}



\begin{table}[ht!]
 \vskip -0.05in
    \centering
    \vspace{0.1cm}
    \scalebox{0.82}{
    \begin{tabular}{c|c|c}
    \toprule
    
     & \text{{\color{blue}DR-Trace}} & \text{{\color{orange}V-Trace} / {\color{green}ReTrace}} \\
     \midrule
      & $\delta^{{\color{blue}DR}}_t {=} r_t + \gamma {\color{blue}V}(s_{t+1}) - {\color{blue}Q}(s_t, a_t)$ & $\delta^{{\color{orange} V}/{\color{green}Q}}_t {=} r_t + \gamma {\color{orange} V}(s_{t+1})/{\color{green}Q}(s_{t+1},a_{t+1}) - {\color{orange} V}(s_t)/{\color{green}Q}(s_t,a_t)$ \\
    \midrule
    $V^{\Tilde{\pi}}$ & $\mathbb{E}_{\mu} [V_t + \sum_{k \geq 0} \gamma^k c_{[t:t+k-1]} \rho_{t+k}  {\color{blue}\delta^{DR}_{t+k}}]$ & $\mathbb{E}_{\mu} [V_t + \sum_{k \geq 0} \gamma^k c_{[t:t+k-1]} \rho_{t+k}  {\color{orange}\delta^{V}_{t+k}} ]$ \\
    \midrule
    $Q^{\Tilde{\pi}}$  & $\mathbb{E}_{\mu}   [Q_t + \sum_{k \geq 0} \gamma^k {\color{blue}c_{[t+1:t+k-1]} ( 1_{\{k=0\}} + 1_{\{k > 0\}} \rho_{t+k}) \delta^{DR}_{t+k}} ]$ & $\mathbb{E}_{\mu}   [Q_t + \sum_{k \geq 0} \gamma^k {\color{green}c_{[t+1:t+k]} \delta^{Q}_{t+k}} ]$  \\ 
    \midrule
    $\nabla \mathcal{J}$ & $\mathbb{E}_{\mu} [\rho_t({\color{blue}Q_t^{\Tilde{\pi}}} - V_t)\nabla \log \pi]$ &  $\mathbb{E}_{\mu} [\rho_t({\color{orange}r_t + V_{t+1}^{\Tilde{\pi}}} - V_t)\nabla \log \pi]$ \\
    \bottomrule 
    \end{tabular}
    }
    \caption{Comparison between DR-Trace and V-Trace/ReTrace.}
    \label{tab:drtrace}
 \vskip -0.05in
\end{table}

{\colorred To enable off-policy training with behavior policy $\mu$}, one choice is to estimate $V^{\Tilde{\pi}}$ and $Q^{\Tilde{\pi}}$ in \eqref{eq:grad_qv} and \eqref{eq:grad_pi} by V-Trace and ReTrace. 
As CASA estimates $(V, Q, \pi)$, applying Doubly Robust ~\citep{dr} is feasible and suitable.
We propose DR-Trace and find the convergence rate and the fixed point of DR-Trace are the same as V-Trace's according to its convergence proof. 
For completeness, we provide DR-Trace and its comparison with V-Trace/ReTrace in Table \ref{tab:drtrace}. 
More details are in Appendix \ref{app:drtrace}.

\section{Experiments}
\label{sec:experiments}


\normalsize



\subsection{Basic Setup}
\label{sec:basic_setup}

We employ a Learner-Actor pipeline \citep{impala} for large-scale training.
{\colorred Motivation and ablation experiments on PPO and R2D2 don't use LSTM, only experiments on CASA+DR-Trace use LSTM \citep{lstm}, which is for comparison with other algorithms.}
We use \textit{burn-in} \citep{r2d2} {\colorred when LSTM is used.}
All estimated values share the same backbone, which is followed by two fully connected layers for each individual head.
We use no intrinsic reward and no entropy regularization in any experiment.
We find that using life information can greatly increase the performance of some games. 
However, to be general, we will not end the episode if life is lost.
All hyperparameters are in Appendix \ref{app:hyperparameters}.

{\colorred For brevity, we denote $\nabla L_V = \mathbb{E}_\pi [(V^\pi - V_\theta)\nabla V_\theta]$, $\nabla L_Q = \mathbb{E}_\pi [(Q^\pi - Q_\theta)\nabla Q_\theta]$ and $\nabla \mathcal{J} = \mathbb{E}_\pi [(Q^\pi - V_\theta)\nabla \log \pi_\theta]$, where expectation is batch-wise average in our implementation. 
When we write $<a, b>$ with $a, b \in \{\nabla L_V, \nabla L_Q, \nabla \mathcal{J}\}$, we firstly calculate batch-wise averaged gradient of $a$ and $b$, then we calculate the angle in-between. 
When we write $\cos<\nabla Q, \nabla \log \pi>$ or $\chi$, we mean $\mathbb{E}_\pi [\cos <\nabla_\theta Q_\theta, \nabla_\theta \log \pi_\theta>]$, which firstly calculates element-wise cosines and then takes batch-wise average.
To avoid numerical problem, we calculate $\frac{x\cdot y}{||x||\cdot||y||}$ by $\frac{x\cdot y}{\max(||x||, 10^{-8})\cdot \max(||y||, 10^{-8})}$.
}

\subsection{Application of CASA on Representative Algorithms}
\label{sec:on_ppo_and_r2d2}

\begin{table}[ht!]
    \centering
    \scalebox{0.86}{
    \begin{math}
        \begin{array}{c|ccc|ccc}
    \toprule
    & \text{PPO} &  & \text{PPO+CASA} & \text{R2D2} & & \text{R2D2+CASA} \\
    \midrule
    & & \multirow{4}{*}{$\Rightarrow$} & (V, A) = (V_\theta, A_\theta) & & \multirow{4}{*}{$\Rightarrow$} & (V, A) = (V_\theta, A_\theta) \\
    \text{Func.} & (V, logit) = (V_\theta, logit_\theta) & & \pi = \text{softmax}(A/\tau) & (V, A) = (V_\theta, A_\theta) & & \pi = \text{softmax}(A/\tau) \\
    \text{Approx.} & \pi = \text{softmax}(logit)& & \Bar{A} = A - sg(\pi) \cdot A & Q = A + V & & \Bar{A} = A - sg(\pi) \cdot A \\
    & & & Q = \Bar{A} + sg(V) & & & Q = \Bar{A} + sg(V) \\
    \midrule
    \text{Gradient} & 0.5 \nabla L_V + \nabla \mathcal{J} & \Rightarrow & 0.5 \nabla L_V + \nabla L_Q + \nabla \mathcal{J}  & \nabla L_Q & \Rightarrow & 0.5 \nabla L_V + \nabla L_Q + \nabla \mathcal{J} \\
    \bottomrule 
    \end{array}
    \end{math}
    }
    \caption{Examples of applying CASA on policy-based methods (PPO) and value-based methods (R2D2).}
    \label{tab:on_ppo_and_r2d2}
    \vskip -0.1in
\end{table}

CASA is applicable to existing algorithms. 
We take PPO and R2D2 for demonstration. 
The application of CASA on PPO is straightforward. 
Applying CASA on R2D2 is impossible as either $\epsilon$-greedy policy or $\arg\max Q$ policy breaks the gradient. 
This problem is the same as calculating the gradients of policy improvement for value-based methods.
We use a surrogate policy $\pi_{surrogate} = \text{softmax}(A / \tau)$, which is discussed in Appendix \ref{app:mtv}. 
Table \ref{tab:on_ppo_and_r2d2} summarizes adjustments of function approximations and training gradients. 

Since PPO+CASA and R2D2+CASA have the same function approximation, recalling the fact that value-based methods improve the policy when a more accurate evaluation is achieved and policy-based methods improve the policy for every step, we can balance the two flexibly with $\chi=1$ by $\alpha_1, \alpha_2, \alpha_3$ in \eqref{eq:grad_all}.

In Figure \ref{fig:mtv}, algorithms with CASA show much higher $\cos(\beta)$ and $\chi$. 
PPO+CASA does more exploration than the original PPO, as the entropy of $\pi$ doesn't easily drop to zero. 
R2D2+CASA tends to distinct the state-action values, where we use the entropy of $Q$ to measure how greedy the current state-action values are. 

\subsection{Behavior of Gradients on different structures}
\label{sec:ablation}

\begin{table}[ht!]
\centering
\begin{minipage}[b]{0.4\textwidth}
\scalebox{0.88}{
\begin{math}
\begin{array}{c|c}
\toprule
\text{PPO+CASA} & Q = A_\theta - sg(\pi_\theta) \cdot A_\theta + sg(V_\theta) \\
\midrule
\text{type 1} & Q = A_\theta - \pi_\theta \cdot A_\theta + sg(V_\theta) \\
\hline
\text{type 2} & Q = A_\theta - sg(\pi_\theta) \cdot A_\theta + V_\theta  \\
\hline
\text{type 3} &  Q = A_\theta + sg(V_\theta)  \\ 
\hline
\text{type 4} & Q = A_\theta + V_\theta  \\ 
\hline
\text{type 5} & Q = Q_\theta  \\
\bottomrule
\end{array}
\end{math}
}
\end{minipage}\hfill
\begin{minipage}[c]{0.56\textwidth}
\setlength{\abovecaptionskip}{0pt}%
\caption{\small Behavior of gradient on different types. Type 1$\&$2 are CASA-like structures, where type 1 removes $sg$ of $\pi$ and type 2 removes $sg$ of $V_\theta$. Type 3$\&$4 are dueling-like structures, where type 3 adds $sg$ to $V$ for dueling-Q and type 4 is dueling-Q. Type 5 uses a new head to estimate $Q_\theta$ separately, which can be considered as an auxiliary task to estimate $Q^\pi$. }
\label{tab:ablation_parameters}
\end{minipage}
\end{table}

Though we show that CASA satisfies $\nabla Q \propto \nabla \log \pi$, {\colorred which means $\chi = 1$}, it's unknown if the structure of CASA is unique. 
As $Q = A - \mathbb{E}_\pi[A] + sg(V)$ is a direct refinement of dueling-DQN, we try several different structures of PPO+CASA. 
All settings of estimating state-action values are shown in Table \ref{tab:ablation_parameters}. 
We always use $0.5 \cdot \nabla L_V + \nabla L_Q + \nabla \mathcal{J}$ as the training gradient. 
We present Breakout and Qbert in Figure \ref{fig:ablation}.


\begin{figure*}[t!]
\centering
\includegraphics[width=1.0\linewidth]{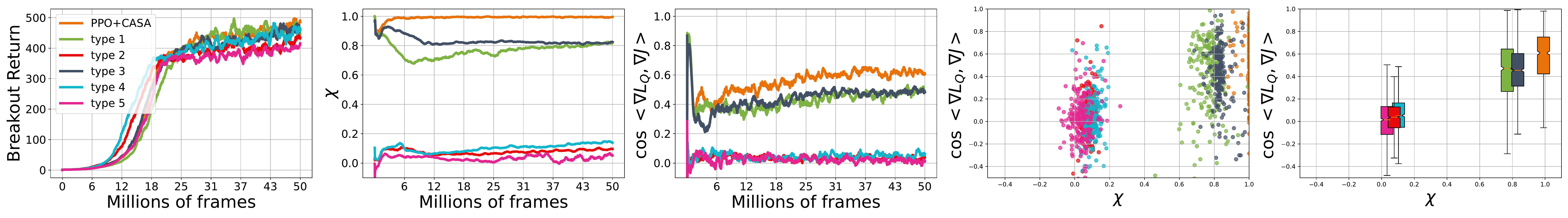}
\includegraphics[width=1.0\linewidth]{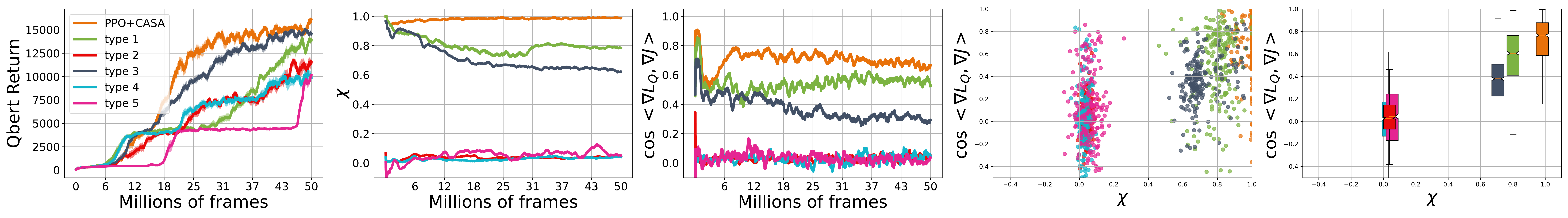}
\caption{The ablation results evaluated on Breakout (top row) and Qbert (bottom row). From left to right is the \emph{return} curve, $\chi$, $\cos(\beta)$, \emph{scatter plot} of $(\chi, \cos(\beta))$ and \emph{box plot} of $(\chi, \cos(\beta))$. Each scatter point is one batch sampled from every consecutive 100 batches. Each box is the interquartile range of scatter points.
}
\label{fig:ablation}
\end{figure*}

For the sake of clarity, we group PPO+CASA and type 3 as $sg\text{-}V$ group, type 2 and type 4 as $no\text{-}sg\text{-}V$ group. 
{\colorred The $sg\text{-}V$ group has higher $\chi$ and higher $\cos(\beta)$, which is closer to the compatible condition and the consistency between two GPI steps,} 
and $no\text{-}sg\text{-}V$ group is always worst than its contrast in $sg\text{-}V$ group. 

{\colorred PPO+CASA has $\chi=1$ and the highest $\cos(\beta)$. }
Type 1 has less returns than PPO+CASA. 
Hence, when applying a CASA-like structure, stopping the gradient of $\pi$ is always preferred. 

{\colorred Type 5 uses an individual head to estimate $Q^\pi$, which performs the worst. 
Hence, a well-designed CASA-like or dueling-like structure is always preferred. }

{\colorred By scatter plot and box plot in Figure \ref{fig:ablation}, $\chi$ and $\cos(\beta)$ are positive correlated depending on different structures.
This phenomenon answers part of the last question of Section \ref{sec:motivation}: for these specific designed structures, $\chi$ and $\cos(\beta)$ show positive correlation. }

\subsection{Evaluation of CASA on Atari Games}
\label{sec:atari_results}

We present an extensive evaluation on CASA, where we train CASA + DR-Trace on 57 Atari games and report the results in terms of two metrics. 
The first is \underline{H}uman \underline{N}ormalized \underline{S}core (HNS), which normalizes the reward by random policy and human expert policies. 
The other is \underline{S}tandardized \underline{A}tari \underline{BE}nchmark for \underline{R}L (SABER), which  normalizes the reward by random policy and human world records, where the normalized score is capped by 200$\%$. 
SABER is considered because recent studies show that the median HNS could easily get hacked by the algorithm since it is sensitive to improvement  on a small subset of games. 
Table \ref{tab:atari_results} summarizes the results.

\begin{table}[h!]
    \centering
    \vskip -0.05in
    \scalebox{1.0}{
    \begin{tabular}{c|cc|cc}
    \toprule

     & Mean HNS  & Median HNS  & Mean SABER  & Median SABER  \\ 
    \midrule
    Rainbow & 873.97 & 230.99 & 28.39 & 4.92 \\
    IMPALA  & 957.34 & 191.82 & 29.45 & 4.31 \\ 
    LASER   & 1741.36 & \textbf{454.91} & \textbf{36.77} & 8.08 \\
   \textbf{CASA}    & \textbf{1941.08}   & 246.36 & 36.10 & \textbf{10.29} \\

    \bottomrule
    \end{tabular}
     }
    \caption{Evaluation scores for the methods on Atari benchmark presented in $\%$. 
    }
    \label{tab:atari_results}
\end{table}

Note that CASA is a variant of IMPALA with DR-Trace, and it achieves substantially better records than IMPALA across all the evaluation metrics. It also scores substantially better than all the methods in terms of mean HNS and median SABER scores. Though off-policy methods are known as  privileged for HNS evaluation due to replay, CASA outperforms strong off-policy baseline Rainbow. Though LASER outperforms CASA in Median HNS and Mean SABER, CASA outperforms it in median SABER and mean HNS. 
Overall, the aforementioned results demonstrate the conflict-averse strategy efficiently boosts the performance in large-scale training scenarios and outperform strong on/off-policy algorithms. Hyperparameters and individual games are presented in Appendix \ref{app:hyperparameters} and Appendix \ref{app:atari_results}, respectively.

\section{Related Works}
\label{sec:relatedworks}

Both value-based or policy-based approaches comply with the principle of GPI,
but two GPI steps are coarsely related to each other such that jointly optimizing both functions might potentially bring conflicts. 
Despite of such crucial issue in GPI with function approximation, most decent model-free algorithms adopt a standard policy improvement/evaluation regime without considering conflict diminishing properties. 
The issue of reducing conflicts among multiple models trained simultaneously was considered in earlier machine learning literature,
such as for robust parameter estimation for multiple estimators under incomplete data~\citep{robins1995semiparametric,lunceford2004stratification,kang2007demystifying} and multitask learning with  gradient similarity measure~\citep{ChenNHLKCA20,YuK0LHF20,JavaloyV22}. 

When the idea is introduced to reinforcement learning, earliest attempts  tackle  conservative and safe policy iteration problems~\citep{KakadeL02,HazanK11,PirottaRPC13}. Recently, more works have emerged to study GPI in a fine-grained manner. In \citep{op_reinforce}, 
a new Bellman operator is introduced which implements GPI with a policy improvement operator and a projection operator, where the projection attempts to find the best approximation of policy among  realizable policies. 
In \citep{RaileanuF21}, the policy and value updates are decoupled by approximating two networks with representation regularization. 
In \citep{CobbeHKS21}, GPI  is separated into a policy improvement and a feature distillation step. On contrast to the aforementioned works, we tackle the conflicts in GPI at the gradient-level, with theoretical analysis. 
Our work is related to \citep{pcl}, which utilizes both the unbiasedness and stability of on-policy training and the data efficiency of off-policy training to form a soft consistency error. Our work bridges the gap between the two GPI steps from an alternative angle of establishing a closer relationship between policy and value functions in their forms, without the focus on off-policy correction.     
{\colorred Due to the difficulty of controlling the gap between GPI steps, we instead consider $\chi$. 
The condition $\chi = 1$ is close to compatible value function \citep{sutton1999policy, kakade2001natural}, shown in Section \ref{sec:motivation} and Appendix \ref{app:comp_v}.}

{\colorred 
\section{Limitation}
\label{sec:limit}
It's noticeable that CASA is only applied on discrete action space for now. 
We further find CASA applicable to any function approximation that is able to estimate advantage functions of all actions.
We provide additional discussion on continuous action space in Appendix \ref{app:cts_space}.

Since $\pi$ shares all parameters of value functions, it brings $\chi=1$ but sacrifices the \textit{freedom} of $\pi$ to be parameterized by other parameters. 
We conjecture that CASA is one endpoint of a trade-off curve between $\chi$ and the \textit{freedom} of $\pi$, where the other endpoint is that $\pi$ shares no parameter with value functions. 
}

\section{Ethics and  Reproducibility Statement }

This paper is aimed at academic issues in deep reinforcement learning, and the experiment used is also in the early stage, but it may provide opportunities for malicious applications of reinforcement learning in the future. We describe all details to reproduce the main
experimental results in Appendix \ref{app:hyperparameters}. 

\section{Conclusion} 
\label{sec:conclusion}
%
{\colorred This paper attempts to eliminate gradient inconsistency between policy improvement and policy evaluation.
The proposed innovative actor-critic design \textbf{C}ritic \textbf{AS} an \textbf{A}ctor (CASA) enhances consistency of two GPI steps by satisfying a weaker compatible condition.} 
We present both theoretical proof and empirical evaluation for CASA.
The results show that our proposed method achieves state-of-the-art performance standards with noticeable performance gain over several strong baselines when evaluated on ALE 200 million (200M) benchmark.
We also present several ablation studies, which demonstrates the effectiveness of the proposed method's theoretical properties. 
{\colorred Future work includes studying the connection between the compatible condition and the gradient consistency between policy improvement and policy evalution.}

\bibliography{iclr2023_conference}

\begin{thebibliography}{35}
\providecommand{\natexlab}[1]{#1}
\providecommand{\url}[1]{\texttt{#1}}
\expandafter\ifx\csname urlstyle\endcsname\relax
  \providecommand{\doi}[1]{doi: #1}\else
  \providecommand{\doi}{doi: \begingroup \urlstyle{rm}\Url}\fi

\bibitem[Badia et~al.(2020)Badia, Piot, Kapturowski, Sprechmann, Vitvitskyi,
  Guo, and Blundell]{agent57}
Adri{\`a}~Puigdom{\`e}nech Badia, Bilal Piot, Steven Kapturowski, Pablo
  Sprechmann, Alex Vitvitskyi, Daniel Guo, and Charles Blundell.
\newblock Agent57: Outperforming the atari human benchmark.
\newblock \emph{arXiv preprint arXiv:2003.13350}, 2020.

\bibitem[Chen \& He(2020)Chen and He]{simsiam}
Xinlei Chen and Kaiming He.
\newblock Exploring simple siamese representation learning.
\newblock \emph{arXiv preprint arXiv:2011.10566}, 2020.

\bibitem[Chen et~al.(2020)Chen, Ngiam, Huang, Luong, Kretzschmar, Chai, and
  Anguelov]{ChenNHLKCA20}
Zhao Chen, Jiquan Ngiam, Yanping Huang, Thang Luong, Henrik Kretzschmar, Yuning
  Chai, and Dragomir Anguelov.
\newblock Just pick a sign: Optimizing deep multitask models with gradient sign
  dropout.
\newblock In \emph{Advances in Neural Information Processing Systems 33: Annual
  Conference on Neural Information Processing Systems 2020, NeurIPS 2020,
  December 6-12, 2020, virtual}, 2020.

\bibitem[Cobbe et~al.(2021)Cobbe, Hilton, Klimov, and Schulman]{CobbeHKS21}
Karl Cobbe, Jacob Hilton, Oleg Klimov, and John Schulman.
\newblock Phasic policy gradient.
\newblock In \emph{Proceedings of the 38th International Conference on Machine
  Learning, {ICML} 2021, 18-24 July 2021, Virtual Event}, volume 139 of
  \emph{Proceedings of Machine Learning Research}, pp.\  2020--2027, 2021.

\bibitem[Dempster et~al.(1977)Dempster, Laird, and Rubin]{em}
A.~P. Dempster, N.~M. Laird, and D.~B. Rubin.
\newblock Maximum likelihood from incomplete data via the em algorithm.
\newblock \emph{JOURNAL OF THE ROYAL STATISTICAL SOCIETY, SERIES B},
  39\penalty0 (1):\penalty0 1--38, 1977.

\bibitem[Espeholt et~al.(2018)Espeholt, Soyer, Munos, Simonyan, Mnih, Ward,
  Doron, Firoiu, Harley, Dunning, et~al.]{impala}
Lasse Espeholt, Hubert Soyer, Remi Munos, Karen Simonyan, Volodymir Mnih, Tom
  Ward, Yotam Doron, Vlad Firoiu, Tim Harley, Iain Dunning, et~al.
\newblock Impala: Scalable distributed deep-rl with importance weighted
  actor-learner architectures.
\newblock \emph{arXiv preprint arXiv:1802.01561}, 2018.

\bibitem[Ghosh et~al.(2020)Ghosh, C~Machado, and Le~Roux]{op_reinforce}
Dibya Ghosh, Marlos C~Machado, and Nicolas Le~Roux.
\newblock An operator view of policy gradient methods.
\newblock \emph{Advances in Neural Information Processing Systems},
  33:\penalty0 3397--3406, 2020.

\bibitem[Hazan \& Kale(2011)Hazan and Kale]{HazanK11}
Elad Hazan and Satyen Kale.
\newblock Better algorithms for benign bandits.
\newblock \emph{J. Mach. Learn. Res.}, 12:\penalty0 1287--1311, 2011.

\bibitem[Hessel et~al.(2017)Hessel, Modayil, Van~Hasselt, Schaul, Ostrovski,
  Dabney, Horgan, Piot, Azar, and Silver]{rainbow}
Matteo Hessel, Joseph Modayil, Hado Van~Hasselt, Tom Schaul, Georg Ostrovski,
  Will Dabney, Dan Horgan, Bilal Piot, Mohammad Azar, and David Silver.
\newblock Rainbow: Combining improvements in deep reinforcement learning.
\newblock \emph{arXiv preprint arXiv:1710.02298}, 2017.

\bibitem[Javaloy \& Valera(2022)Javaloy and Valera]{JavaloyV22}
Adri{\'{a}}n Javaloy and Isabel Valera.
\newblock Rotograd: Gradient homogenization in multitask learning.
\newblock In \emph{The Tenth International Conference on Learning
  Representations, {ICLR} 2022, Virtual Event, April 25-29, 2022}, 2022.

\bibitem[Jiang \& Li(2016)Jiang and Li]{dr}
Nan Jiang and Lihong Li.
\newblock Doubly robust off-policy value evaluation for reinforcement learning.
\newblock In \emph{International Conference on Machine Learning}, pp.\
  652--661. PMLR, 2016.

\bibitem[Kakade(2001)]{kakade2001natural}
Sham~M Kakade.
\newblock A natural policy gradient.
\newblock \emph{Advances in neural information processing systems}, 14, 2001.

\bibitem[Kakade \& Langford(2002)Kakade and Langford]{KakadeL02}
Sham~M. Kakade and John Langford.
\newblock Approximately optimal approximate reinforcement learning.
\newblock In \emph{Machine Learning, Proceedings of the Nineteenth
  International Conference {(ICML} 2002), University of New South Wales,
  Sydney, Australia, July 8-12, 2002}, pp.\  267--274, 2002.

\bibitem[Kang \& Schafer(2007)Kang and Schafer]{kang2007demystifying}
Joseph~DY Kang and Joseph~L Schafer.
\newblock Demystifying double robustness: A comparison of alternative
  strategies for estimating a population mean from incomplete data.
\newblock \emph{Statistical science}, 22\penalty0 (4):\penalty0 523--539, 2007.

\bibitem[Kapturowski et~al.(2018)Kapturowski, Ostrovski, Quan, Munos, and
  Dabney]{r2d2}
Steven Kapturowski, Georg Ostrovski, John Quan, Remi Munos, and Will Dabney.
\newblock Recurrent experience replay in distributed reinforcement learning.
\newblock In \emph{International conference on learning representations}, 2018.

\bibitem[Lunceford \& Davidian(2004)Lunceford and
  Davidian]{lunceford2004stratification}
Jared~K Lunceford and Marie Davidian.
\newblock Stratification and weighting via the propensity score in estimation
  of causal treatment effects: a comparative study.
\newblock \emph{Statistics in medicine}, 23\penalty0 (19):\penalty0 2937--2960,
  2004.

\bibitem[Mnih et~al.(2015)Mnih, Kavukcuoglu, Silver, Rusu, Veness, Bellemare,
  Graves, Riedmiller, Fidjeland, Ostrovski, et~al.]{dqn}
Volodymyr Mnih, Koray Kavukcuoglu, David Silver, Andrei~A Rusu, Joel Veness,
  Marc~G Bellemare, Alex Graves, Martin Riedmiller, Andreas~K Fidjeland, Georg
  Ostrovski, et~al.
\newblock Human-level control through deep reinforcement learning.
\newblock \emph{nature}, 518\penalty0 (7540):\penalty0 529--533, 2015.

\bibitem[Munos et~al.(2016)Munos, Stepleton, Harutyunyan, and
  Bellemare]{retrace}
Remi Munos, Tom Stepleton, Anna Harutyunyan, and Marc Bellemare.
\newblock Safe and efficient off-policy reinforcement learning.
\newblock In D.~D. Lee, M.~Sugiyama, U.~V. Luxburg, I.~Guyon, and R.~Garnett
  (eds.), \emph{Advances in Neural Information Processing Systems 29}, pp.\
  1054--1062. Curran Associates, Inc., 2016.

\bibitem[Nachum et~al.(2017)Nachum, Norouzi, Xu, and Schuurmans]{pcl}
Ofir Nachum, Mohammad Norouzi, Kelvin Xu, and Dale Schuurmans.
\newblock Bridging the gap between value and policy based reinforcement
  learning.
\newblock In \emph{Advances in Neural Information Processing Systems}, pp.\
  2775--2785, 2017.

\bibitem[Pedersen(2019)]{ftw}
Carsten~Lund Pedersen.
\newblock Re: Human-level performance in 3d multiplayer games with
  population-based reinforcement learning.
\newblock \emph{Science}, 2019.

\bibitem[Pirotta et~al.(2013)Pirotta, Restelli, Pecorino, and
  Calandriello]{PirottaRPC13}
Matteo Pirotta, Marcello Restelli, Alessio Pecorino, and Daniele Calandriello.
\newblock Safe policy iteration.
\newblock In \emph{Proceedings of the 30th International Conference on Machine
  Learning, {ICML} 2013, Atlanta, GA, USA, 16-21 June 2013}, volume~28 of
  \emph{{JMLR} Workshop and Conference Proceedings}, pp.\  307--315, 2013.

\bibitem[Raileanu \& Fergus(2021)Raileanu and Fergus]{RaileanuF21}
Roberta Raileanu and Rob Fergus.
\newblock Decoupling value and policy for generalization in reinforcement
  learning.
\newblock In \emph{Proceedings of the 38th International Conference on Machine
  Learning, {ICML} 2021, 18-24 July 2021, Virtual Event}, volume 139 of
  \emph{Proceedings of Machine Learning Research}, pp.\  8787--8798, 2021.

\bibitem[Robins \& Rotnitzky(1995)Robins and
  Rotnitzky]{robins1995semiparametric}
James~M Robins and Andrea Rotnitzky.
\newblock Semiparametric efficiency in multivariate regression models with
  missing data.
\newblock \emph{Journal of the American Statistical Association}, 90\penalty0
  (429):\penalty0 122--129, 1995.

\bibitem[Schaul et~al.(2015)Schaul, Quan, Antonoglou, and Silver]{priority_q}
Tom Schaul, John Quan, Ioannis Antonoglou, and David Silver.
\newblock Prioritized experience replay.
\newblock \emph{arXiv preprint arXiv:1511.05952}, 2015.

\bibitem[Schmidhuber(1997)]{lstm}
Sepp Hochreiter;~Jürgen Schmidhuber.
\newblock Long short-term memory.
\newblock \emph{Neural Computation.}, 1997.

\bibitem[Schmitt et~al.(2020)Schmitt, Hessel, and Simonyan]{laser}
Simon Schmitt, Matteo Hessel, and Karen Simonyan.
\newblock Off-policy actor-critic with shared experience replay.
\newblock In \emph{International Conference on Machine Learning}, pp.\
  8545--8554. PMLR, 2020.

\bibitem[Schulman et~al.(2015)Schulman, Levine, Abbeel, Jordan, and
  Moritz]{trpo}
John Schulman, Sergey Levine, Pieter Abbeel, Michael Jordan, and Philipp
  Moritz.
\newblock Trust region policy optimization.
\newblock In \emph{International conference on machine learning}, pp.\
  1889--1897, 2015.

\bibitem[Schulman et~al.(2017)Schulman, Wolski, Dhariwal, Radford, and
  Klimov]{ppo}
John Schulman, Filip Wolski, Prafulla Dhariwal, Alec Radford, and Oleg Klimov.
\newblock Proximal policy optimization algorithms.
\newblock \emph{arXiv preprint arXiv:1707.06347}, 2017.

\bibitem[Sutton(1988)]{Sutton88lambda}
Richard~S. Sutton.
\newblock Learning to predict by the methods of temporal differences.
\newblock \emph{Mach. Learn.}, 3:\penalty0 9--44, 1988.

\bibitem[Sutton \& Barto(2018)Sutton and Barto]{sutton}
Richard~S Sutton and Andrew~G Barto.
\newblock \emph{Reinforcement learning: An introduction}.
\newblock MIT press, 2018.

\bibitem[Sutton et~al.(1999)Sutton, McAllester, Singh, and
  Mansour]{sutton1999policy}
Richard~S Sutton, David McAllester, Satinder Singh, and Yishay Mansour.
\newblock Policy gradient methods for reinforcement learning with function
  approximation.
\newblock \emph{Advances in neural information processing systems}, 12, 1999.

\bibitem[Toromanoff et~al.(2019)Toromanoff, Wirbel, and Moutarde]{saber}
Marin Toromanoff, Emilie Wirbel, and Fabien Moutarde.
\newblock Is deep reinforcement learning really superhuman on atari? leveling
  the playing field.
\newblock \emph{arXiv preprint arXiv:1908.04683}, 2019.

\bibitem[Vinyals et~al.(2019)Vinyals, Babuschkin, Czarnecki, Mathieu, Dudzik,
  Chung, Choi, Powell, Ewalds, Georgiev, et~al.]{alpha_star}
Oriol Vinyals, Igor Babuschkin, Wojciech~M Czarnecki, Micha{\"e}l Mathieu,
  Andrew Dudzik, Junyoung Chung, David~H Choi, Richard Powell, Timo Ewalds,
  Petko Georgiev, et~al.
\newblock Grandmaster level in starcraft ii using multi-agent reinforcement
  learning.
\newblock \emph{Nature}, 575\penalty0 (7782):\penalty0 350--354, 2019.

\bibitem[Wang et~al.(2016)Wang, Schaul, Hessel, Hasselt, Lanctot, and
  Freitas]{dueling_q}
Ziyu Wang, Tom Schaul, Matteo Hessel, Hado Hasselt, Marc Lanctot, and Nando
  Freitas.
\newblock Dueling network architectures for deep reinforcement learning.
\newblock In \emph{International conference on machine learning}, pp.\
  1995--2003, 2016.

\bibitem[Yu et~al.(2020)Yu, Kumar, Gupta, Levine, Hausman, and Finn]{YuK0LHF20}
Tianhe Yu, Saurabh Kumar, Abhishek Gupta, Sergey Levine, Karol Hausman, and
  Chelsea Finn.
\newblock Gradient surgery for multi-task learning.
\newblock In \emph{Advances in Neural Information Processing Systems 33: Annual
  Conference on Neural Information Processing Systems 2020, NeurIPS 2020,
  December 6-12, 2020, virtual}, 2020.

\end{thebibliography}
\bibliographystyle{iclr2023_conference}
\newpage
\appendix

{\colorred 
\section{Compatible Value Function}
\label{app:comp_v}

The original policy gradient with compatible value function is stated as follow. 
\begin{theorem}
[\cite{sutton1999policy}]
Let $Q_w$ be a state-action function with parameter $w$ and $\pi_\theta$ be a policy function with parameter $\theta$. 
If $Q_w$ satisfies $\mathbb{E}_{\pi} [(Q^\pi - Q_w) \nabla_w Q_w] = 0$ and 
$\nabla_w Q_w = \nabla_\theta \log \pi_\theta,$
then $$\nabla_\theta \mathcal{J} = \mathbb{E}_\pi [Q_w \nabla_\theta \log \pi_\theta].$$
\label{thm:pg_fa}
\end{theorem}
If we let $w = \theta$ in Theorem \ref{thm:pg_fa}, where $Q_w$ and $\pi_\theta$ share parameters, we have the following theorem. 
\begin{theorem}
Let $Q_\theta$ be a state-action function with parameter $\theta$ and $\pi_\theta$ be a policy function with parameter $\theta$. 
If $Q_\theta$ satisfies $\mathbb{E}_{\pi} [(Q^\pi - Q_\theta) \nabla_\theta Q_\theta] = 0$ and 
$\nabla_\theta Q_\theta = \nabla_\theta \log \pi_\theta,$
then $$\nabla_\theta \mathcal{J} = \mathbb{E}_\pi [Q_\theta \nabla_\theta \log \pi_\theta].$$
\label{thm:pg_fa2}
\end{theorem}
Define 
$$\chi \overset{def}{=} \mathbb{E}_\pi [\cos <\nabla_\theta Q_\theta, \nabla_\theta \log \pi_\theta>].$$
We show that $\chi = 1$ is the necessary condition for the compatible condition $\nabla_\theta Q_\theta = \nabla_\theta \log \pi_\theta$. 
\begin{theorem}
i) If $\nabla_\theta Q_\theta \propto \nabla_\theta \log \pi_\theta$ for all states, then $\chi = 1$.

ii) If $\chi = 1$, then $\nabla_\theta Q_\theta \propto \nabla_\theta \log \pi_\theta$ for all states. 
\label{thm:connect_cond}
\end{theorem}
By Theorem \ref{thm:connect_cond}, $\chi = 1$ is equivalent to $\nabla_\theta Q_\theta \propto \nabla_\theta \log \pi_\theta$, and $\nabla_\theta Q_\theta \propto \nabla_\theta \log \pi_\theta$ is the necessary condition for $\nabla_\theta Q_\theta = \nabla_\theta \log \pi_\theta$, hence $\chi = 1$ is the necessary condition for $\nabla_\theta Q_\theta = \nabla_\theta \log \pi_\theta$.
\begin{proof}
i) Since $\nabla_\theta Q_\theta \propto \nabla_\theta \log \pi_\theta$, we have $<\nabla_\theta Q_\theta, \nabla_\theta \log \pi_\theta> = 0$. 
By definition of $\chi$, we have 
$$\chi = \mathbb{E}_\pi [\cos <\nabla_\theta Q_\theta, \nabla_\theta \log \pi_\theta>] = \mathbb{E}_\pi [1] = 1.$$

ii) Since $\chi \leq 1$ and $\cos(x)$ is monotonic decreasing as $x$ goes from $0$ to $\pi$, the equality $\chi = 1$ only holds when all states satisfy $\cos <\nabla_\theta Q_\theta, \nabla_\theta \log \pi_\theta> = 0$, which means $\nabla_\theta Q_\theta \propto \nabla_\theta \log \pi_\theta$. 
\end{proof}
}

\clearpage

\section{Gradients Between Policy Improvement and Policy Evaluation}
\label{app:mtv}

\begin{table}[hb!]
    \centering
    \scalebox{0.90}{
    \begin{math}
        \begin{array}{c|c|c|c}
    \toprule
     & \text{Function Approximation} & \text{Train Gradients} & \text{Cosine of Interested Angles} \\
    \midrule
    
    \text{PPO} & (V, logit) = (V_\theta, logit_\theta) & 0.5 \nabla L_V + \nabla \mathcal{J} & 
    \\ 
    & \pi = \text{softmax}(logit) & & \\
    
    \midrule
    
    \text{PPO ver.1} & (Q, logit) = (Q_\theta, logit_\theta), & 0.5 \nabla L_V + \nabla \mathcal{J} & \cos<\nabla L_Q, \nabla \mathcal{J}>
    \\
    & \pi = \text{softmax}(logit) & & \cos<\nabla Q, \nabla \log \pi> \\
    & V = sg(\pi)\cdot Q & & 
    \\
    
    \midrule
    
    \text{PPO ver.2} & (Q, logit) = (Q_\theta, logit_\theta), & 0.5 \nabla L_V + \nabla L_Q + \nabla \mathcal{J} & \cos<\nabla L_Q, \nabla \mathcal{J}> 
    \\
    & pi = \text{softmax}(logit) & & \cos<\nabla Q, \nabla \log \pi> \\
    & V = sg(\pi)\cdot Q & & 
    \\
    
    \midrule
    
    \text{PPO+CASA} & (V, A) = (V_\theta, A_\theta), & 0.5 \nabla L_V + \nabla L_Q + \nabla \mathcal{J} & \cos<\nabla L_Q, \nabla \mathcal{J}> 
    \\
    & \pi = \text{softmax}(A/\tau), & & \cos<\nabla Q, \nabla \log \pi> \\
    & \Bar{A} = A - sg(\pi) \cdot A & & 
    \\
    & Q = \Bar{A} + sg(V) & &  \\
    
    \bottomrule 
    \end{array}
    \end{math}
    }
    
    \caption{PPO is the original PPO. PPO ver.1 and PPO ver.2 are adapted versions to calculate $\nabla L_Q$. PPO+CASA is applying CASA on PPO, which is described in Sec. \ref{sec:on_ppo_and_r2d2}.}
    \label{tab:ppo_mtv}
\end{table}

\begin{table}[ht!]
    \centering
    \scalebox{0.90}{
    \begin{math}
        \begin{array}{c|c|c|c}
    \toprule
     & \text{Function Approximation} & \text{Train Gradients} & \text{Cosine of Interested Angles} \\
    \midrule
    
    \text{R2D2} & (V, A) = (V_\theta, A_\theta) & \nabla L_Q & \cos<\nabla L_Q, \nabla \mathcal{J}>  
    \\
    & Q = A + V & & \\
    & \pi = \text{softmax}(A / \tau) & & 
    \\

    \midrule
    
    \text{R2D2 ver.1} & (V, A) = (V_\theta, A_\theta) & 0.5 \nabla L_V + \nabla L_Q & \cos<\nabla L_Q, \nabla \mathcal{J}>  
    \\
    & Q = A + V & & 
    \\
    & \pi = \text{softmax}(A / \tau) & & 
    \\

    \midrule
    
    \text{R2D2+CASA} & (V, A) = (V_\theta, A_\theta), & 0.5 \nabla L_V + \nabla L_Q + \nabla \mathcal{J} & \cos<\nabla L_Q, \nabla \mathcal{J}>  
    \\
    & \pi = \text{softmax}(A/\tau),  & & 
    \\
    & \Bar{A} = A - sg(\pi) \cdot A & & 
    \\
    & Q = \Bar{A} + sg(V) & &  \\ 
    \bottomrule 
    \end{array}
    \end{math}
     }
    \caption{R2D2 is the original R2D2. R2D2 ver.1 is adapted version to include $\nabla L_V$ for training. R2D2+CASA is applying CASA on R2D2, which is described in Sec. \ref{sec:on_ppo_and_r2d2}.}
    \label{tab:r2d2_mtv}
\end{table}

To understand the behavior of 
{\colorred $$
    \beta \overset{def}{=} <\mathbb{E}_\pi[(Q^\pi-Q_\theta)\nabla_\theta Q_\theta],\, \mathbb{E}_\pi[(Q^\pi-V_\theta) \nabla_\theta \log \pi_\theta]>
$$
}
and 
{\colorred 
$$\chi \overset{def}{=} \mathbb{E}_\pi [\cos <\nabla_\theta Q_\theta, \nabla_\theta \log \pi_\theta>]$$
}
in reinforcement learning algorithms, we choose PPO as a representative for policy-based methods and R2D2 as a representative for value-based algorithms. 

Define $$L_V(\theta) = \mathbb{E}_\pi [ (V^{\pi} - V_\theta)^2 ],\  L_Q(\theta) = \mathbb{E}_\pi [ (Q^{\pi} - Q_\theta)^2 ],$$
and $$\nabla_\theta \mathcal{J}(\theta) = \mathbb{E}_\pi \left[ (Q^{\pi}  - V_\theta ) \nabla_\theta \log \pi \right].$$
We usually have above three kinds of loss functions in reinforcement learning, which aim to estimate the state values, state-action values and the policy. 
We do not talk about the estimations of $V^\pi$ and $Q^\pi$ as they are estimated as their usual way of PPO's and R2D2's. 
All hyperparameters are listed in Appendix \ref{app:hyperparameters}. 

{\colorred For brevity, we write 
$$\cos<\nabla Q, \nabla \log \pi> = \mathbb{E}_\pi [\cos <\nabla_\theta Q_\theta, \nabla_\theta \log \pi_\theta>],$$
and
$$
\begin{aligned}
    &\cos<\nabla L_Q, \nabla \mathcal{J}> = \cos<\mathbb{E}_\pi[(Q^\pi-Q_\theta)\nabla_\theta Q_\theta],\, \mathbb{E}_\pi[(Q^\pi-V_\theta) \nabla_\theta \log \pi_\theta]>, \\
    &\cos<\nabla L_V, \nabla \mathcal{J}> = \cos<\mathbb{E}_\pi[(V^\pi-V_\theta)\nabla_\theta V_\theta],\, \mathbb{E}_\pi[(Q^\pi-V_\theta) \nabla_\theta \log \pi_\theta]>, \\
    &\cos<\nabla L_V, \nabla L_Q> = \cos<\mathbb{E}_\pi[(V^\pi-V_\theta)\nabla_\theta V_\theta],\, \mathbb{E}_\pi[(Q^\pi-Q_\theta) \nabla_\theta Q_\theta]>. \\
\end{aligned}
$$}

The fact that PPO only has $\nabla_\theta L_V$ and $\nabla_\theta \mathcal{J}$ and R2D2 only has $\nabla_\theta L_Q$ is the main difficulty to track $\cos(\beta)$ and $\chi$. 
To solve the problem, we adjust PPO and R2D2 with different versions.

For PPO, we displace the estimation of $V_\theta$ by $sg(\pi)\cdot Q_\theta$, where $Q_\theta$ is estimated by function approximation and $V_\theta$ is estimated by taking the expectation of $Q_\theta$.
All versions of PPO are listed in Table \ref{tab:ppo_mtv}.

For R2D2, we point out that though we apply $\epsilon$-greedy to interact with environments, $\epsilon$ is only used for exploration and the final target policy of value-based methods is simply $\arg\max Q_\theta$. 
Because $\arg\max Q_\theta$ breaks the gradient, we use a surrogate policy to approximate the gradient of policy improvement. 
Since R2D2 uses dueling structure and $\text{softmax}(A_\theta / \tau) = \text{softmax}(Q_\theta / \tau) \overset{\tau \rightarrow 0+}{\longrightarrow} \arg\max Q_\theta$, we use $\pi_{surrogate} = \text{softmax}(A_\theta / \tau)$ to calculate the policy gradient. 
We only use $\pi_{surrogate}$ on learner to calculate the gradient, where the policy that interacts with environments is still $\epsilon$-greedy. 
All versions of R2D2 are listed in Table \ref{tab:r2d2_mtv}.








\clearpage

{\colorred \section{On Discussing Application of CASA on Continuous Action Space}
\label{app:cts_space}

As we can see CASA is only applied to discrete action space in the main context, we make a discussion on whether CASA is applicable on continuous action space. 
For brevity, we let $\tau=1$ and write \eqref{eq:casa} as:
\begin{equation}
\left\{
    \begin{aligned}
        &\pi = \text{softmax}(A), \\
        &\Bar{A} = A - \mathbb{E}_{\pi} [A], \\
        &Q = \Bar{A} + sg(V).
    \end{aligned}
\right. 
\end{equation}
The difficulty comes from estimating two quantities, one is $\text{softmax}(A)$, the other is $\mathbb{E}_{\pi} [A]$. 
This comes from the fact that discrete action space is countable so these two quantities are expressed in a closed-form, while continuous action space is uncountable so an accurate estimation of these two quantities is intractable. 
We can surely apply Monte Carlo methods to approximate, but a more elegant close-form expression may be preferred. 
Then this becomes another problem: \textit{how to estimate (state-action values / advantages / policy probabilities) of all actions in a continuous action space efficiently without loss of generality?}
This is another representational design problem, which is out of scope of this paper, so we don't touch much about it. 
But with the hope of inspiring a better solution to this problem, we provide one practical way of applying CASA on continuous action space based on kernel-based machine learning. 

Let $a_0, \dots, a_k$ to be basis actions in the action space. 
Let $A(s, a_0), \dots, A(s, a_k)$ to be advantage functions for tuples of states and basis actions. 
They can either share parameters or be isolated. 
Let $K(\cdot, \cdot)$ be a kernel function defined on the product of two action spaces. 
For any $a$ in the action space, we can estimate $A(s, a)$ by a decomposition such like $$A(s, a) = \frac{1}{Z_a} (K(a_0, a) A(s, a_0) + \dots + K(a_k, a) A(s, a_k)),$$ where $Z_a = \sum_{i=0}^k K(a_i, a)$ is a normalization constant. 

Since $K(\cdot, a)$ is a closed-form function of $a$, and $|\{A(s, a_0), \dots, A(s, a_k)\}|$ is finite, we can make a closed-form expression of both $\text{softmax}(A)$ and $\mathbb{E}_{\pi} [A]$. 
Then we can apply CASA directly on this expression, with one function estimates $V$ and the other function estimates advantages of all actions in a closed-form with only state as input.  
The policy is defined directly by $\text{softmax}$ of all advantages. 
In details, we define
\begin{equation}
\left\{
    \begin{aligned}
        &\pi(s, a) = \exp (A(s, a)) / \int_{a} \exp (A(s, a)) da, \\
        &\Bar{A}(s, a) = A(s, a) - \int_{a} sg(\pi(s, a)) A(s, a) da, \\
        &Q(s, a) = \Bar{A}(s, a) + sg(V(s)).
    \end{aligned}
\right. 
\end{equation}

Then it satisfies the consistency of CASA on continuous action space.
$$
\begin{aligned}
    \nabla \log \pi(s, a) &= \nabla A(s, a) - \frac{\nabla \int_{a} \exp (A(s, a)) da}{\int_{a} \exp (A(s, a)) da} \\
    &= \nabla A(s, a) - \frac{ \int_{a} \exp (A(s, a)) \nabla A(s, a) da}{\int_{a} \exp (A(s, a)) da} \\
    &= \nabla A(s, a) - \int_{a} \frac{  \exp (A(s, a)) }{\int_{a} \exp (A(s, a)) da} \nabla A(s, a) da \\
    &= \nabla A(s, a) - \int_{a} \pi(s, a) \nabla A(s, a) da \\
    &= \nabla \Bar{A}(s, a) = \nabla Q(s, a). 
\end{aligned}
$$
}

\section{DR-Trace}
\label{app:drtrace}


As CASA estimates $(V, Q, \pi)$, we would ask
\textbf{i)} how to guarantee that $\Tilde{\pi}_{VTrace} = \Tilde{\pi}_{ReTrace}$, 
\textbf{ii)} how to exploit $(V, Q, \pi)$ to make a better estimation. 
Though we can apply V-Trace to estimate $V$ and ReTrace to estimate $Q$ with proper hyperparameters to guarantee $\Tilde{\pi}_{VTrace} = \Tilde{\pi}_{ReTrace}$, it's more reasonable to estimate $(V, Q)$ together. 
Inspired by Doubly Robust, which is shown to maximally reduce the variance, we introduce DR-Trace, which estimates $V$ by 
$$
\label{eq:dr-v}
    \begin{aligned}
        V_{DR}^{\Tilde{\pi}} (s_t) &\overset{def}{=} \mathbb{E}_{\mu} [ 
        V(s_t) + \sum_{k \geq 0} \gamma^k 
     c_{[t:t+k-1]} \rho_{t+k}  \delta^{DR}_{t+k} ],  
    \end{aligned}
$$
{\colorred where $\mu$ is the behavior policy}, $\delta^{DR}_t \overset{def}{=} r_t + \gamma V(s_{t+1}) - Q(s_t, a_t)$ is one-step Doubly Robust error, $\rho_t \overset{def}{=} \min\{\frac{\pi_t}{\mu_t}, \Bar{\rho} \}$ and $c_t \overset{def}{=} \min\{\frac{\pi_t}{\mu_t}, \Bar{c}\}$ are clipped per-step importance sampling, $c_{[t: t+k]} \overset{def}{=} \prod_{i=0}^{k} c_{t+i}$.

With one step Bellman equation, we estimate $Q$ by
$$
\label{eq:dr-q}
    \begin{aligned}
         Q_{DR}^{\Tilde{\pi}} (s_t, a_t) 
         &\overset{def}{=} \mathbb{E}_{s_{t+1}, r_t \sim p(\cdot, \cdot | s_t, a_t)} [  r_t + \gamma   V_{DR}^{\Tilde{\pi}} (s_{t+1}) ] 
        \\
        &=  \mathbb{E}_{\mu}   [
        Q(s_t, a_t) + \sum_{k \geq 0}  \gamma^k
        c_{[t+1:t+k-1]} \Tilde{\rho}_{t, k}
        \delta^{DR}_{t+k}
        ], 
    \end{aligned}
$$
where $\Tilde{\rho}_{t, k} =  1_{\{k=0\}} + 1_{\{k > 0\}} \rho_{t+k}$.

\begin{theorem}
    Define $\Bar{A} = A - \mathbb{E}_\pi[A]$, $Q = \Bar{A} + sg(V)$,
    $$
    \begin{aligned}
    &\mathscr{T}(Q) \overset{def}{=} \mathbb{E}_{\mu}   [
        Q(s_t, a_t) + \sum_{k \geq 0}  \gamma^k
        c_{[t+1:t+k-1]} \Tilde{\rho}_{t, k}
        \delta^{DR}_{t+k}
        ], \\
    &\mathscr{S}(V) \overset{def}{=} \mathbb{E}_{\mu}   [
        V(s_t) + \sum_{k \geq 0}  \gamma^k
        c_{[t:t+k-1]} \rho_{t, k}
        \delta^{DR}_{t+k}
        ], \\
    &\mathscr{U}(Q, V) = (\mathscr{T}(Q) - \mathbb{E}_\pi[Q] + \mathscr{S}(V), \mathscr{S}(V)), \\
    &\mathscr{U}^{(n)}(Q, V) = \mathscr{U}(\mathscr{U}^{(n-1)}(Q, V)),
    \end{aligned}
    $$
    then $\mathscr{U}^{(n)}(Q, V) \rightarrow (Q^{\Tilde{\pi}}, V^{\Tilde{\pi}})$ that corresponds to 
    $$
        \Tilde{\pi}(a|s) = \frac
        {\min \left\{\Bar{\rho} \mu (a|s), \pi(a|s)\right\}}
        {\sum_{b \in \mathcal{A}}\min \left\{\Bar{\rho} \mu (b|s), \pi(b|s)\right\}}.
    $$ as $n \rightarrow +\infty$.
\label{thm:dr}
\end{theorem}
\begin{proof}
    See Appendix \ref{app:proof}, Theorem \ref{thm_app:dr}.
\end{proof}
Theorem \ref{thm:dr} shows that DR-Trace is a contraction mapping and $(V, Q)$ converges to $(V^{\Tilde{\pi}}, Q^{\Tilde{\pi}})$ that corresponds to 
$$
    \begin{aligned}
        \Tilde{\pi}(a|s) = \frac
        {\min \left\{\Bar{\rho} \mu (a|s), \pi(a|s)\right\}}
        {\sum_{b \in \mathcal{A}}\min \left\{\Bar{\rho} \mu (b|s), \pi(b|s)\right\}}.
    \end{aligned}
$$

\begin{figure}[h]
    \centering
\includegraphics[width=\linewidth]{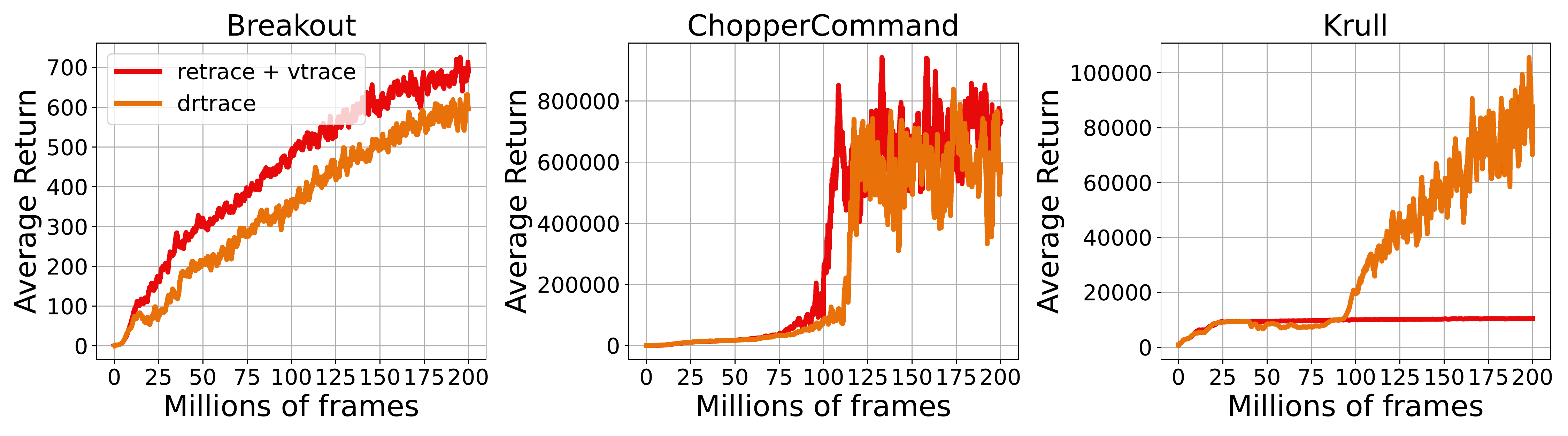}
    \caption{Ablation study for w/wo DR-Trace on Breakout, ChopperCommand and Krull.}
    \label{fig:app_dr_trace}
\end{figure}

According to our proof, DR-Trace should work similar to V-Trace and ReTrace, as the convergence rate and the limitation are same. 
We compare DR-Trace with V-Trace+ReTrace in Figure \ref{fig:app_dr_trace}, where we replace estimation of state values by V-Trace and estimation of state-action values by ReTrace. 
We call V-Trace+ReTrace as No-DR-Trace for brevity. 
No-DR-Trace performs better on Breakout and ChopperCommand, but fails to make a breakthrough on Krull. 
Recalling the fact that Doubly Robust can maximally reduce the variance of Bellman error, No-DR-Trace is less stable but also potential to achieve a better performance. 
A conclusion cannot be made about No-DR-Trace, as this phenomenon means that No-DR-Trace is less stable than DR-Trace, but it also holds the potential to achieve a better performance.


\section{Proofs}
\label{app:proof}

\theoremstyle{plain}
\newtheorem{Lemma_app}{Lemma}[section]
\newtheorem{Theorem_app}{Theorem}[section]
\theoremstyle{definition}
\newtheorem*{Remark_app}{Remark}
\theoremstyle{remark}






\begin{Lemma_app}
(i) Define $\pi = softmax(A / \tau)$, then $\nabla \log \pi = (\textbf{1} - \pi) \frac{\nabla A}{\tau}$. 
(ii) Denote $sg$ to be stop gradient and define $\Bar{A} = A - \mathbb{E}_\pi [A]$, $Q = \Bar{A} + sg(V)$, then $\nabla Q = (\textbf{1} - \pi) \nabla A$.
\label{lemma_app:vannila_grad}
\end{Lemma_app}
\begin{proof}

As $Q = \Bar{A} + sg(V) = A - sg(\pi)\cdot A + sg(V)$, it's obvious that $\nabla Q = (\textbf{1} - \pi) \nabla A$.

For $\log \pi$, it's a standard derivative of cross entropy, so we have $\nabla \log \pi = (\textbf{1} - \pi) \nabla (A / \tau) = (\textbf{1} - \pi) \frac{\nabla A}{\tau}$.
\end{proof}

\begin{Lemma_app}
Define $\Bar{A}= A - \mathbb{E}_\pi[A]$, $Q = \Bar{A} + sg(V), \pi = softmax(A / \tau)$, then 
$$
\mathbb{E}_\pi \left[ (Q - V) \nabla \log \pi \right]
= - \tau \nabla \textbf{H}[\pi].
$$
\label{lemma_app:eqiv_pg_ent}
\end{Lemma_app}
\begin{proof}
Since 
$$
\pi = \exp(A / \tau) / Z,\ Z = \int_\mathcal{A} \exp(A / \tau),
$$
we have 
$$
A = \tau \log \pi + \tau \log Z.
$$
Based on the observation that $\mathbb{E}_\pi \left[ f(s) \nabla \log \pi (\cdot | s) \right] = 0$, 
we have 
$$\mathbb{E}_\pi \left[ \mathbb{E}_\pi[A] \cdot \nabla \log \pi \right] = 0,$$ 
$$\mathbb{E}_\pi \left[ \log Z \cdot \nabla \log \pi \right] = 0.$$

On the one hand,
$$
\begin{aligned}
    \mathbb{E}_\pi \left[ (Q - V) \nabla \log \pi \right]
    &= \mathbb{E}_\pi \left[ A \nabla \log \pi \right] 
    - \mathbb{E}_\pi \left[ \mathbb{E}_\pi[A] \cdot \nabla \log \pi \right] \\
    &= \tau \mathbb{E}_\pi \left[ \log \pi \nabla \log \pi \right]
    + \tau \mathbb{E}_\pi \left[ \log Z \cdot \nabla \log \pi \right] \\
    &= \tau \mathbb{E}_\pi \left[ \log \pi \nabla \log \pi \right].
\end{aligned}
$$

On the other hand, 
$$
\begin{aligned}
    \nabla \textbf{H} [\pi] 
    &= - \nabla \int_\mathcal{A} \pi_i \log \pi_i \\
    &= - \int_\mathcal{A}  \nabla \pi_i \cdot \log \pi_i - \int_\mathcal{A} \pi_i \nabla \log \pi_i  \\
    &= - \int_\mathcal{A}  \pi_i \nabla \log \pi_i \cdot \log \pi_i - \int_\mathcal{A}  \pi_i \frac{\nabla \pi_i}{\pi_i} \\
    &= - \mathbb{E}_\pi \left[ \log \pi \nabla \log \pi \right].
\end{aligned}
$$
Hence, $
\mathbb{E}_\pi \left[ (Q - V) \nabla \log \pi \right]
= - \tau \nabla \textbf{H}[\pi]
$.
\end{proof}

\begin{Theorem_app}
    Define $\Bar{A} = A - \mathbb{E}_\pi[A]$, $Q = \Bar{A} + sg(V)$.
    Define $$
    \begin{aligned}
    &\mathscr{T}(Q) \overset{def}{=} \mathbb{E}_{\mu}   [
        Q(s_t, a_t) + \sum_{k \geq 0}  \gamma^k
        c_{[t+1:t+k-1]} \Tilde{\rho}_{t, k}
        \delta^{DR}_{t+k}
        ], \\
    &\mathscr{S}(V) \overset{def}{=} \mathbb{E}_{\mu}   [
        V(s_t) + \sum_{k \geq 0}  \gamma^k
        c_{[t:t+k-1]} \rho_{t, k}
        \delta^{DR}_{t+k}
        ], \\
    &\mathscr{U}(Q, V) = (\mathscr{T}(Q) - \mathbb{E}_\pi[Q] + \mathscr{S}(V), \mathscr{S}(V)), \\
    &\mathscr{U}^{(n)}(Q, V) = \mathscr{U}(\mathscr{U}^{(n-1)}(Q, V)),
    \end{aligned}
    $$
    then $\mathscr{U}^{(n)}(Q, V) \rightarrow (Q^{\Tilde{\pi}}, V^{\Tilde{\pi}})$ that corresponds to 
    $$
        \Tilde{\pi}(a|s) = \frac
        {\min \left\{\Bar{\rho} \mu (a|s), \pi(a|s)\right\}}
        {\sum_{b \in \mathcal{A}}\min \left\{\Bar{\rho} \mu (b|s), \pi(b|s)\right\}}.
    $$ as $n \rightarrow +\infty$.
\label{thm_app:dr}
\end{Theorem_app}
\begin{Remark_app}
$\mathscr{T}(Q) - \mathbb{E}_\pi[Q] + \mathscr{S}(V)$ is \textbf{exactly} how $Q$ is updated at training time. 
Since $Q = \Bar{A} + sg(V)$, if we apply gradient ascent on $Q$ and $V$ in directions $\nabla L_Q(\theta)$ and $\nabla L_V(\theta)$ respectively, change of $Q$ comes from two aspects. One comes from $\nabla L_Q(\theta)$, which changes $A$, the other comes from $\nabla L_V(\theta)$, which changes $V$. Because the gradient of $V$ is stopped when estimating $Q$, the latter is captured by "minus old baseline, add new baseline", which is $- \mathbb{E}_\pi[Q] + \mathscr{S}(V)$ in Theorem \ref{thm_app:dr}.
\end{Remark_app}
\begin{proof}
 Define
 $$
 \begin{aligned}
        \widetilde{\mathscr{T}}(Q) &= - \mathbb{E}_\pi[Q] + \mathscr{T}(Q), \\
        \widetilde{\mathscr{U}}(Q, V) &= (\widetilde{\mathscr{T}}(Q), \mathscr{S}(V)), \\
        \widetilde{\mathscr{U}}^{(n)}(Q, V) &=   \widetilde{\mathscr{U}}(\widetilde{\mathscr{U}}^{(n-1)}(Q, V)).
 \end{aligned}
 $$
By Lemma \ref{lemma_app:dr_q}, $\widetilde{\mathscr{T}}^{(n)}(Q)$ converges to some $A^*$ as $n \rightarrow \infty$. This process will not influence the estimation of $V$ as the gradient of $V$ is stopped when estimating $Q$. According to the proof, $A^*$ does not depend on $V$. \\
By Lemma \ref{lemma_app:dr_v}, $\mathscr{S}^{(n)}(V)$ converges to some $V^*$ as $n \rightarrow \infty$. \\
Hence, we have
$$
\widetilde{\mathscr{U}}^{(n)}(Q, V) \rightarrow (A^*, V^*)\ \ as\ \ n \rightarrow +\infty. 
$$
By definition, 
$$
\mathscr{U}(Q, V) = (\widetilde{\mathscr{T}}(Q) + \mathscr{S}(V), \mathscr{S}(V)),
$$
we can regard $\widetilde{\mathscr{T}}(Q) + \mathscr{S}(V)$ as $Q$ and regard $\mathscr{S}(V)$ as $V$, then
$$
\begin{aligned}
    \mathscr{U}^{(2)}(Q, V) 
    &= \mathscr{U}(\widetilde{\mathscr{T}}(Q) + \mathscr{S}(V), \mathscr{S}(V)) \\
    &= (\mathscr{T}(\widetilde{\mathscr{T}}(Q) + \mathscr{S}(V)) -\mathscr{S}(V) + \mathscr{S}^{(2)}(V), \mathscr{S}^{(2)}(V)) \\
    &= (\widetilde{\mathscr{T}}^{(2)}(Q) + \mathscr{S}^{(2)}(V), \mathscr{S}^{(2)}(V)).
\end{aligned}
$$
By induction, 
$$
\begin{aligned}
    \mathscr{U}^{(n)}(Q, V) &= (\widetilde{\mathscr{T}}^{(n)}(Q) + \mathscr{S}^{(n)}(V), \mathscr{S}^{(n)}(V)) \\
    &\rightarrow (A^*+V^*, V^*)\ \ as\ \ n\rightarrow + \infty.
\end{aligned}
$$
Same as \citep{impala}, 
$$
    \Tilde{\pi}(a|s) = \frac
    {\min \left\{\Bar{\rho} \mu (a|s), \pi(a|s)\right\}}
    {\sum_{b \in \mathcal{A}}\min \left\{\Bar{\rho} \mu (b|s), \pi(b|s)\right\}}.
$$ 
is the policy s.t. the Bellman equation holds, which is 
$$\mathbb{E}_\mu[\rho_t (r_t + \gamma V_{t+1} - V_t) | \mathscr{F}_t] = 0,$$ and $\mathscr{U}(Q^{\Tilde{\pi}}, V^{\Tilde{\pi}}) = (Q^{\Tilde{\pi}}, V^{\Tilde{\pi}})$. \\
So we have
$(A^*+V^*, V^*) = (Q^{\Tilde{\pi}}, V^{\Tilde{\pi}}).$
\end{proof}

\begin{Lemma_app}
Define $\Bar{A}= A - \mathbb{E}_\pi[A]$, $Q = \Bar{A} + sg(V)$,
then operator 
$$
    \mathscr{T}(Q) \overset{def}{=} \mathbb{E}_{\mu}   [
        Q(s_t, a_t) + \sum_{k \geq 0}  \gamma^k
        c_{[t+1:t+k-1]} \Tilde{\rho}_{t, k}
        \delta^{DR}_{t+k}
        ]
$$
is a contraction mapping w.r.t. $Q$.
\label{lemma_app:dr_q}
\end{Lemma_app}
\begin{Remark_app}
Note that $\mathscr{T}(Q)$ is exactly \eqref{eq:dr-q}. 

Since $Q = A + sg(V)$, the gradient of $V$ is stopped when estimating $Q$, updating $Q$ will not change $V$, which is equivalent to updating $A$.
Without loss of generality, we assume $V$ is fixed as $V^*$ in the proof.
\end{Remark_app}
\begin{proof}

$\Bar{A} = A - \mathbb{E}_\pi[A]$ shows $\mathbb{E}_\pi[\Bar{A}] = 0$, which guarantees that no matter how we update $A$, we always have $\mathbb{E}_\pi[Q] = V^*$.

Based on above observations, define 
$$
    \widetilde{\mathscr{T}}(Q) \overset{def}{=} - \mathbb{E}_\pi [Q] + \mathscr{T}(Q).
$$

It's obvious that we only need to prove $\widetilde{\mathscr{T}}(Q)$ is a contraction mapping.

For brevity, we denote $$Q_t = Q(s_t, a_t), A_t = A(s_t, a_t), V^*_t = V^*(s_t).$$

Noticing that $\Tilde{\rho}_{t, 0} = 1$, let $\mathscr{F}$ represent filtration, we can rewrite $\widetilde{\mathscr{T}}$ as 
\begin{equation}
\label{eq:dr_a_2}
\begin{aligned}
    \widetilde{\mathscr{T}}(Q)
    &= \mathbb{E}_{\mu}   [
        A_t + \sum_{k \geq 0}  \gamma^k
        c_{[t+1:t+k-1]} \Tilde{\rho}_{t, k}
        \delta^{DR}_{t+k}
        ] \\
    &= \mathbb{E}_{\mu}   [
        -V^*_t + \sum_{k \geq 0}  \gamma^k
        c_{[t+1:t+k-1]} \Tilde{\rho}_{t, k}
        r_{t+k}
        + 
        \sum_{k \geq 0}  \gamma^{k+1}
        c_{[t+1:t+k-1]} \Delta_k ],
        \\
\end{aligned}
\end{equation}
where 
\begin{equation}
\label{eq:dr_delta}
    \Delta_k = \mathbb{E}_{\mu}\left[\Tilde{\rho}_{t, k} V^*_{t+k+1} - c_{t+k} \Tilde{\rho}_{t, k+1} Q_{t+k+1} | \mathscr{F}_{t+k}\right].
\end{equation}
By definition of $Q$,
$$
    \mathbb{E}_{\mu}[V_{t+k+1}^*|\mathscr{F}_{t+k}] 
    = \mathbb{E}_{\mu}[
    \mathbb{E}_\pi[Q_{t+k+1}|\mathscr{F}_{t+k+1}]
    |\mathscr{F}_{t+k}], \\
$$
we can rewrite \eqref{eq:dr_delta} as
\begin{equation}
\label{eq:dr_q_delta}
\Delta_k = \mathbb{E}_{\mu}[
(
\Tilde{\rho}_{t, k} \frac{\pi_{t+k+1}}{\mu_{t+k+1}}- c_{t+k} \Tilde{\rho}_{t, k+1} 
) Q_{t+k+1} | \mathscr{F}_{t+k}
].
\end{equation}
For any $Q_1 = A_1 + sg(V^*)$, $Q_2 = A_2 + sg(V^*)$, since
$$
\mathbb{E}_{\mu}[
(
\Tilde{\rho}_{t, k} \frac{\pi_{t+k+1}}{\mu_{t+k+1}}- c_{t+k} \Tilde{\rho}_{t, k+1} 
) | \mathscr{F}_{t+k}
] \geq 0,
$$
by \eqref{eq:dr_a_2} \eqref{eq:dr_q_delta}, we have 
$$
        || \widetilde{\mathscr{T}}(Q_1) - \widetilde{\mathscr{T}}(Q_2) || 
        \leq \mathcal{C} || Q_1 - Q_2 ||,
$$
where 
$$
    \begin{aligned}
        \mathcal{C} 
        &= \mathbb{E}_{\mu} [ \sum_{k \geq 0}  \gamma^{k+1} c_{[t+1:t+k-1]} 
        (
        \Tilde{\rho}_{t, k} \frac{\pi_{t+k+1}}{\mu_{t+k+1}}- c_{t+k} \Tilde{\rho}_{t, k+1} 
        ) ]
        \\
        &= \mathbb{E}_{\mu} [1 -1 + \sum_{k \geq 0}  \gamma^{k+1} c_{[t+1:t+k-1]} 
        \left(
        \Tilde{\rho}_{t, k} - c_{t+k} \Tilde{\rho}_{t, k+1} 
        \right) ] 
        \\
        &= 1 - (1 - \gamma)  \mathbb{E}_{\mu} [\sum_{k \geq 0} \gamma^{k}c_{[t+1:t+k-1]} \Tilde{\rho}_{t, k}  ] \\
        &\leq 1 - (1 - \gamma) < 1.
    \end{aligned}
$$
Hence, $\widetilde{\mathscr{T}}(Q)$ is a contraction mapping and converges to some fixed function, which we denote as $A^*$. So $\mathscr{T}(Q)$ is also a contraction mapping and converges to $A^*+V^*$.
\end{proof}

\begin{Lemma_app}
Define $Q = A + sg(V)$ with $\mathbb{E}_\pi [A] = 0$,
then operator 
$$
    \mathscr{S}(V) \overset{def}{=} \mathbb{E}_{\mu}  [
        V(s_t) + \sum_{k \geq 0}  \gamma^k
        c_{[t:t+k-1]} \rho_{t, k}
        \delta^{DR}_{t+k}
        ]
$$
is a contraction mapping w.r.t. $V$.
\label{lemma_app:dr_v}
\end{Lemma_app}
\begin{Remark_app}
Note that $\mathscr{S}(V)$ is exactly \eqref{eq:dr-v}. 
\end{Remark_app}
\begin{proof}

Same as Lemma \ref{lemma_app:dr_q}, we can get
$$
    \Delta_k = \mathbb{E}_{\mu}\left[
    \left( \rho_{t+k} - c_{t+k} \rho_{t+k+1}\right) V_{t+k+1} 
     -  c_{t+k} \rho_{t+k+1} A^*_{t+k+1} | \mathscr{F}_{t+k}\right],
$$
so we have 
$$
    \Delta^1_k - \Delta^2_k = \mathbb{E}_{\mu}\left[ 
    \left( \rho_{t+k} - c_{t+k} \rho_{t+k+1}\right) \cdot  
   (V^1_{t+k+1} -  V^2_{t+k+1})
     | \mathscr{F}_{t+k}\right].
$$
The remaining proof is identical to \citep{impala}'s.
\end{proof}

\clearpage

\section{Hyperparameters}
\label{app:hyperparameters}

Our python packages are shown in Table \ref{tab:package}.

\begin{table}[h!]
\begin{center}
\begin{tabular}{l@{\hspace{.43cm}}l@{\hspace{.22cm}}}
\toprule
\textbf{Package} & \textbf{Version}  \\
\midrule
ale-py & 0.6.0.dev20200207 \\
gym & 0.19.0 \\
tensorflow & 1.15.2 \\
opencv-python & 4.1.2.30 \\
opencv-contrib-python & 4.4.0.46 \\
\bottomrule
\end{tabular}
\caption{Versions for python packages among all experiments.}
\label{tab:package}
\end{center}
\end{table}

All experiments follow the shared hyperparameters as in Table \ref{tab:shared_hyperparameters}. 
The specific hyperparameters for PPO, R2D2 and CASA+DR-Trace are shown in Table \ref{tab:ppo_hyperparameters}, Table \ref{tab:r2d2_hyperparameters} and Table \ref{tab:drtrace_hyperparameters}.
The only exceptions are $V$-loss scaling, $Q$-loss scaling and $\pi$-loss scaling, which may be zero depending on some specific ablation settings. 
We will state these three hyperparameters every time in all experiments.

\begin{table}[H]
\begin{center}
\scalebox{0.95}{
\begin{tabular}{l@{\hspace{.43cm}}l@{\hspace{.22cm}}}
\toprule
\textbf{Parameter} & \textbf{Value}  \\
\midrule
Atari Version & NoFrameskip-v4 \\
Atari Wrapper & gym.wrappers.atari\_preprocessing \\
Image Size & (84, 84) \\
Grayscale & Yes \\
Num. Action Repeats & 4 \\
Num. Frame Stacks & 4 \\
Action Space & Full \\
End of Episode When Life Lost & No \\
Num. Environments & 160 \\
Random No-ops & 30 \\
Burn-in Stored Recurrent State & Yes \\
Bootstrap & Yes \\
Optimizer & Adam Weight Decay \\
Weight Decay Rate & 0.01 \\
Weight Decay Schedule & Anneal linearly to 0 \\
Learning Rate & 5e-4 \\
Warmup Steps & 4000 \\
Learning Rate Schedule & Anneal linearly to 0 \\
AdamW $\beta_1$ & 0.9 \\
AdamW $\beta_2$ & 0.98 \\
AdamW $\epsilon$ & 1e-6 \\
AdamW Clip Norm & 50.0 \\
Learner Push Model Every $n$ Steps & 25 \\
Actor Pull Model Every $n$ Steps & 64 \\
\bottomrule
\end{tabular}}
\caption{Configurations for shared hyperparameters among all experiments.}
\label{tab:shared_hyperparameters}
\end{center}
\end{table}


\clearpage


\begin{table}[H]
\begin{center}
\scalebox{0.85}{
\begin{tabular}{l@{\hspace{.43cm}}l@{\hspace{.22cm}}}
\toprule
\textbf{Parameter} & \textbf{Value}  \\
\midrule
{\colorred Num. States} & {\colorred 50M} \\
Sample Reuse & 1 \\
Reward Shape & clip$(r, 0, 1)$ \\
{\colorred Burn-in} & {\colorred 0} \\
{\colorred Seq-length} & {\colorred 40} \\
Discount ($\gamma$) & 0.995 \\
{\colorred Batch size} & {\colorred 8} \\
{\colorred Backbone} & {\colorred IMPALA,shallow without LSTM} \\
PPO clip $\epsilon$ & 0.2 \\
GAE $\lambda$ & 0.8 \\
Temperature ($\tau$) & 0.1 \\
\bottomrule
\end{tabular}}
\caption{Hyperparameter configurations for PPO.}
\label{tab:ppo_hyperparameters}
\end{center}
\end{table}

\begin{table}[H]
\begin{center}
\scalebox{0.85}{
\begin{tabular}{l@{\hspace{.43cm}}l@{\hspace{.22cm}}}
\toprule
\textbf{Parameter} & \textbf{Value}  \\
\midrule
{\colorred Num. States} & {\colorred 50M} \\
Sample Reuse & 2 \\
Target Shape & $Q_{t}^{\Tilde{\pi}} = h(\sum_{i=0}^{n-1} \gamma^i r_{t+i} + \gamma^n h^{-1}(\text{Double}(Q_{t+n})))$ \\
Target Shape Function $h$ & $h(x) = \text{sign}(x) \cdot (\sqrt{|x| + 1} - 1) + 10^{-3} x$ \\
Bootstrap Length $n$ & 5 \\
$\epsilon$-greedy & $\epsilon \sim 0.4^{\text{uniform}(1, 8)}$ \\
PER Sample Temperature $\alpha$ & 0.9 \\
PER Buffer Size & 400000 \\
{\colorred Burn-in} & {\colorred 0} \\
{\colorred Seq-length} & {\colorred 40} \\
Discount ($\gamma$) & 0.997 \\
{\colorred Batch size} & {\colorred 8} \\
{\colorred Backbone} & {\colorred IMPALA,shallow without LSTM} \\
Temperature ($\tau$) & 0.1 \\
\bottomrule
\end{tabular}}
\caption{Hyperparameter configurations for R2D2.}
\label{tab:r2d2_hyperparameters}
\end{center}
\end{table}

\begin{table}[H]
\begin{center}
\scalebox{0.85}{
\begin{tabular}{l@{\hspace{.43cm}}l@{\hspace{.22cm}}}
\toprule
\textbf{Parameter} & \textbf{Value}  \\
\midrule
{\colorred Num. States} & {\colorred 200M} \\
Sample Reuse & 2 \\
Reward Shape & $\log (|r| + 1.0) \cdot (2 \cdot 1_{\{r \geq 0\}} - 1_{\{r < 0\}})$ \\
{\colorred Burn-in} & {\colorred 40} \\
{\colorred Seq-length} & {\colorred 80} \\
Discount ($\gamma$) & 0.997 \\
{\colorred Batch size} & {\colorred 64} \\
{\colorred Backbone} & {\colorred IMPALA,deep} \\
{\colorred LSTM Units} & {\colorred 256} \\
$V$-loss Scaling ($\alpha_1$) & 1.0 \\
$Q$-loss Scaling ($\alpha_2$) & 10.0 \\
$\pi$-loss Scaling ($\alpha_3$) & 10.0 \\
Temperature ($\tau$) & 1.0 \\
Importance Sampling Clip $\Bar{c}$ & 1.05 \\
Importance Sampling Clip $\Bar{\rho}$ & 1.05 \\
\bottomrule
\end{tabular}}
\caption{Hyperparameter configurations for CASA + DR-Trace.}
\label{tab:drtrace_hyperparameters}
\end{center}
\end{table}
\clearpage

\section{Evaluation of CASA on Atari Games}
\label{app:atari_results}

Random scores and average human's scores are from \citep{agent57}.
Human World Records (HWR) are from \citep{saber}.
Rainbow's scores are from \citep{rainbow}.
IMPALA's scores are from \citep{impala}.
LASER's scores are from \citep{laser}, no sweep at 200M. 


\tiny
\begin{center}
\hskip -0.05in
\scalebox{1.05}{
\begin{tabular}{ccccccccccc}
\toprule
Games & RND & HUMAN & RAINBOW & HNS(\%) & IMPALA & HNS(\%) & LASER & HNS(\%) & CASA & HNS(\%) \\
\midrule
Scale  &     &       & 200M   &       &  200M    &        & 200M   &
       &  200M   &  \\
\midrule
 alien  & 227.8 & 7127.8 & 9491.7 & 134.26 & 15962.1  & 228.03 & \textbf{35565.9} & \textbf{512.15} & 26137 & 375.50 \\
 amidar & 5.8   & 1719.5 & \textbf{5131.2} & \textbf{299.08} & 1554.79  & 90.39  & 1829.2  & 106.4  & 560   & 32.34 \\
 assault & 222.4 & 742   & 14198.5 & 2689.78 & 19148.47 & 3642.43  & \textbf{21560.4} & \textbf{4106.62} & 16228  & 3080.37  \\
 asterix & 210   & 8503.3 & \textbf{428200} & \textbf{5160.67} & 300732   & 3623.67  & 240090  & 2892.46 & 213580 & 2572.80 \\
 asteroids & 719 & 47388.7 & 2712.8 & 4.27   & 108590.05 & 231.14  & \textbf{213025}  &  \textbf{454.91} & 80339   & 170.60 \\
 atlantis & 12850 & 29028.1 & 826660 & 5030.32 & 849967.5 & 5174.39 & 841200 & 5120.19 & \textbf{3211600} & \textbf{19772.10} \\
 bank heist & 14.2 & 753.1  & \textbf{1358}   & \textbf{181.86}  & 1223.15  & 163.61  & 569.4  & 75.14   & 895.3   & 119.24 \\
 battle zone & 236 & 37187.5 & 62010 & 167.18  & 20885    & 55.88  & 64953.3 & 175.14  & \textbf{91269}   & \textbf{246.36} \\
 beam rider & 363.9 & 16926.5 & 16850.2 & 99.54 & 32463.47 & 193.81 & \textbf{90881.6} & \textbf{546.52} & 57456   & 344.70 \\
 berzerk & 123.7 & 2630.4  & 2545.6   & 96.62  & 1852.7   & 68.98  & \textbf{25579.5}  & \textbf{1015.51} & 1648   & 60.81 \\
 bowling & 23.1 & 160.7   & 30   & 5.01        & 59.92    & 26.76  & 48.3    & 18.31   & \textbf{162.4}     & \textbf{101.24} \\
 boxing  & 0.1  & 12.1    & 99.6 & 829.17      & 99.96    & 832.17 & \textbf{100}   & \textbf{832.5}     & 98.3   & 818.33 \\
 breakout & 1.7 & 30.5    & 417.5 & 1443.75    & \textbf{787.34}   & \textbf{2727.92} & 747.9 & 2590.97  & 624.3  & 2161.81 \\
 centipede & 2090.9 & 12017 & 8167.3 & 61.22   & 11049.75 & 90.26   & \textbf{292792} & \textbf{2928.65} & 102600 & 1012.57 \\
 chopper command & 811 & 7387.8 & 16654 & 240.89 & 28255  & 417.29  & \textbf{761699} & \textbf{11569.27} & 616690 & 9364.42 \\
 crazy climber & 10780.5 & 36829.4 & \textbf{168788.5} & \textbf{630.80} & 136950 & 503.69 & 167820  & 626.93 & 161250 & 600.70 \\
 defender & 2874.5 & 18688.9 & 55105 & 330.27 & 185203 & 1152.93 & 336953  & 2112.50   & \textbf{421600} & \textbf{2647.75} \\
 demon attack & 152.1 & 1971 & 111185 & 6104.40 & 132826.98 & 7294.24 & 133530 & 7332.89 & \textbf{291590} & \textbf{16022.76} \\
 double dunk & -18.6 & -16.4 & -0.3   & 831.82  & -0.33     & 830.45  & 14     & 1481.82 & \textbf{20.25} & \textbf{1765.91} \\
 enduro      & 0   & 860.5 & 2125.9 & 247.05  & 0       & 0.00     & 0    & 0.00       & \textbf{10019} & \textbf{1164.32} \\
 fishing derby & -91.7 & -38.8 & 31.3 & 232.51  & 44.85   & 258.13    & 45.2   & 258.79  & \textbf{53.24} & \textbf{273.99} \\
 freeway       & 0     & 29.6  & \textbf{34} & \textbf{114.86}  & 0     & 0.00       & 0    & 0.00       & 3.46   & 11.69 \\
 frostbite     & 65.2  & 4334.7 & \textbf{9590.5} & \textbf{223.10} & 317.75 & 5.92     & 5083.5 & 117.54  & 1583 & 35.55 \\
 gopher  & 257.6 & 2412.5 & 70354.6 & 3252.91    & 66782.3 & 3087.14 & 114820.7 & 5316.40 & \textbf{188680} & \textbf{8743.90} \\
 gravitar & 173 & 3351.4  & 1419.3  & 39.21   & 359.5      & 5.87    & 1106.2   & 29.36   & \textbf{4311}  & \textbf{130.19} \\
 hero   & 1027 & 30826.4 & \textbf{55887.4} & \textbf{184.10}   & 33730.55  & 109.75   & 31628.7 & 102.69   & 24236 & 77.88 \\
 ice hockey & -11.2 & 0.9 & 1.1    & 101.65   & 3.48      & 121.32   & \textbf{17.4}    & \textbf{236.36}   & 1.56  & 105.45 \\
 jamesbond  & 29    & 302.8 & 19809 & 72.24   & 601.5     & 209.09   & \textbf{37999.8} & \textbf{13868.08} & 12468 & 4543.10 \\
 kangaroo   & 52    & 3035 & \textbf{14637.5} & \textbf{488.05} & 1632    & 52.97    & 14308   & 477.91     & 5399 & 179.25 \\
 krull     & 1598   & 2665.5 & 8741.5  & 669.18 & 8147.4  & 613.53   & 9387.5  &  729.70  & \textbf{64347} & \textbf{5878.13} \\
 kung fu master & 258.5 & 22736.3 & 52181 & 230.99 & 43375.5 & 191.82 & \textbf{607443} & \textbf{2701.26}  & 124630.1 & 553.31 \\
 montezuma revenge & 0  & \textbf{4753.3}  & 384   & 8.08   & 0       & 0.00   & 0.3    & 0.01     & 2488.4  & 52.35 \\
 ms pacman  & 307.3 & 6951.6   & 5380.4  & 76.35   & 7342.32 & 105.88 & 6565.5 & 94.19    & \textbf{7579}  & \textbf{109.44} \\
 name this game & 2292.3 & 8049 & 13136 & 188.37   & 21537.2 & 334.30 & 26219.5 & 415.64  & \textbf{32098} & \textbf{517.76} \\
 phoenix & 761.5 & 7242.6  & 108529 & 1662.80   & 210996.45  & 3243.82 & \textbf{519304} & \textbf{8000.84} & 498590 & 7681.23 \\
 pitfall & -229.4 & \textbf{6463.7} & 0      & 3.43      & -1.66      & 3.40    & -0.6   & 3.42    & -17.8 & 3.16 \\
 pong    & -20.7  & 14.6   & 20.9   & 117.85    & 20.98      & 118.07  & \textbf{21}     &  \textbf{118.13} & 20.39  & 116.40 \\
 private eye & 24.9 & \textbf{69571.3} & 4234 & 6.05     & 98.5       & 0.11    & 96.3   & 0.10    & 134.1  & 0.16 \\
 qbert  & 163.9 & 13455.0 & 33817.5  & 253.20   & \textbf{351200.12}  & \textbf{2641.14} & 21449.6 & 160.15 & 27371 & 204.70 \\
 riverraid & 1338.5 & 17118.0 & 22920.8 & 136.77 & 29608.05  & 179.15  & \textbf{40362.7} & \textbf{247.31} & 11182 & 62.38 \\
 road runner & 11.5 & 7845    & 62041   & 791.85 & 57121     & 729.04  & 45289   & 578.00 & \textbf{251360} & \textbf{3208.64} \\
 robotank   & 2.2   & 11.9  & 61.4   & 610.31    & 12.96     & 110.93  & \textbf{62.1}    & \textbf{617.53} & 10.44  & 84.95 \\
 seaquest  & 68.4 & \textbf{42054.7} & 15898.9 & 37.70    & 1753.2    & 4.01    & 2890.3  & 6.72   & 11862  & 28.09 \\
 skiing & -17098  & \textbf{-4336.9} & -12957.8 & 32.44  & -10180.38 & 54.21   & -29968.4 & -100.86 & -12730 & 34.23 \\
 solaris & 1236.3 & \textbf{12326.7} & 3560.3  & 20.96  & 2365      & 10.18   & 2273.5   & 9.35    & 2319 & 9.76 \\
 space invaders & 148 & 1668.7 & 18789 & 1225.82 & 43595.78 & 2857.09 & \textbf{51037.4} & \textbf{3346.45} & 3031 & 189.58 \\
 star gunner & 664 & 10250 & 127029    & 1318.22 & 200625   & 2085.97 & 321528  & 3347.21 & \textbf{337150} & \textbf{3510.18} \\
 surround    & -10 & 6.5   & \textbf{9.7}       & \textbf{119.39}  & 7.56     & 106.42  & 8.4     & 111.52  & -10  & 0.00 \\
 tennis  & -23.8   & -8.3 & 0        & 153.55    & 0.55     & 157.10  & \textbf{12.2}    & \textbf{232.26}  & -21.05 & 17.74 \\
 time pilot & 3568 & 5229.2 & 12926 & 563.36     & 48481.5  & 2703.84 & \textbf{105316}  & \textbf{6125.34} & 84341 & 4862.62 \\
 tutankham  & 11.4 & 167.6  & 241   & 146.99     & 292.11   & 179.71  & 278.9   & 171.25  & \textbf{381} & \textbf{236.62} \\
 up n down  & 533.4 & 11693.2 & 125755 & 1122.08 & 332546.75 & 2975.08 & 345727 & 3093.19 & \textbf{416020} & \textbf{3723.06} \\
 venture    & 0     & \textbf{1187.5}  & 5.5    & 0.46    & 0         & 0.00    & 0      & 0.00    & 0  & 0.00 \\
 video pinball & 0 & 17667.9  & 533936.5 & 3022.07 & \textbf{572898.27} & \textbf{3242.59} & 511835 & 2896.98 & 297920 & 1686.22 \\
 wizard of wor & 563.5 & 4756.5 & 17862.5 & 412.57 & 9157.5    & 204.96  & \textbf{29059.3} & \textbf{679.60} & 26008 & 606.83 \\
 yars revenge & 3092.9 & 54576.9 & 102557 & 193.19 & 84231.14  & 157.60 & \textbf{166292.3} & \textbf{316.99} & 118730 & 224.61 \\
 zaxxon       & 32.5   & 9173.3 & 22209.5 & 242.62 & 32935.5   & 359.96 & 41118    & 449.47 & \textbf{46070.8}  & \textbf{503.66} \\
\hline
MEAN HNS(\%) &     0.00 & 100.00   &         & 873.97 &         & 957.34  &        & 1741.36 &      & 1941.08 \\
\hline
MEDIAN HNS(\%) & 0.00   & 100.00   &         & 230.99 &         & 191.82  &        & 454.91  &      & 246.36 \\
\bottomrule
\end{tabular}
}
\end{center}
\normalsize
\clearpage

\tiny
\begin{center}
\begin{tabular}{ccccccccccc}
\toprule
Games & RND & HWR & RAINBOW & SABER(\%) & IMPALA & SABER(\%) & LASER & SABER(\%) & CASA & SABER(\%) \\
\midrule
Scale  &     &       & 200M   &       &  200M    &        & 200M   & &  200M   &  \\
\midrule
 alien              & 227.8     & \textbf{251916}    & 9491.7   &3.68    & 15962.1    & 6.25       & 976.51  & 14.04                                & 26137             & 10.29    \\
 amidar             & 5.8       & \textbf{104159}    & 5131.2   &4.92    & 1554.79    & 1.49       & 1829.2  & 1.75                                 & 560             & 0.53            \\
 assault            & 222.4     & 8647               & 14198.5  &165.90  & 19148.47   & 200.00     & \textbf{21560.4} & \textbf{200.00}                               & 16228             & 189.99   \\
 asterix            & 210       & \textbf{1000000}   & 428200   &42.81   & 300732     & 30.06      & 240090  & 23.99                                & 213580            & 21.34  \\
 asteroids          & 719       & \textbf{10506650}  & 2712.8   &0.02    & 108590.05  & 1.03       & 213025  & 2.02                                 & 80339            & 0.76   \\
 atlantis           & 12850     & \textbf{10604840}  & 826660   &7.68    & 849967.5   & 7.90       & 841200  & 7.82                                 & 3211600               & 30.20   \\
 bank heist         & 14.2      & \textbf{82058}     & 1358     &1.64    & 1223.15    & 1.47       & 569.4   & 0.68                                 & 895.3             & 1.07 \\
 battle zone        & 236       &\textbf{801000}    & 62010    &7.71    & 20885      & 2.58       & 64953.3 & 8.08                                           & 91269            & 11.37  \\
 beam rider         & 363.9     & \textbf{999999}    & 16850.2  &1.65    & 32463.47   & 3.21       & 90881.6 & 9.06                                 & 57456           & 5.71    \\
 berzerk            & 123.7     & \textbf{1057940}            & 2545.6   &0.23    & 1852.7     & 0.16       & 25579.5 & 2.41                        & 1648             & 0.14        \\
 bowling            & 23.1      & \textbf{300}       & 30       &2.49    & 59.92      & 13.30      & 48.3    & 9.10                                 & 162.4            & 50.31   \\
 boxing             & 0.1       & \textbf{100}                & 99.6     &99.60   & 99.96      & 99.96      & \textbf{100}     & \textbf{100.00}    & 98.3             & 98.3  \\
 breakout           & 1.7       & \textbf{864}                & 417.5    &48.22   & 787.34     & 91.11      & 747.9   & 86.54                       & 624.3             & 72.20  \\
 centipede          & 2090.9    & \textbf{1301709}   & 8167.3   &0.47    & 11049.75   & 0.69       & 292792  & 22.37                                & 102600           & 7.73 \\
 chopper command    & 811       & \textbf{999999}             & 16654    &1.59    & 28255      & 2.75       & 761699  & 76.15                       & 616690            & 61.64 \\
 crazy climber      & 10780.5   & \textbf{219900}    & 168788.5 &75.56   & 136950     & 60.33      & 167820  & 75.10                                         & 161250           & 71.95       \\
 defender           & 2874.5    & \textbf{6010500}   & 55105    &0.87    & 185203     & 3.03       & 336953  & 5.56                                 & 421600           & 6.97       \\
 demon attack       & 152.1     & \textbf{1556345}   & 111185   &7.13    & 132826.98  & 8.53       & 133530  & 8.57                                 & 291590           & 18.73       \\
 double dunk        & -18.6     & \textbf{21}                 & -0.3     &46.21   & -0.33      & 46.14      & 14      & 82.32                                & 20.25            & 98.11 \\
 enduro             & 0         & 9500               & 2125.9   &22.38   & 0          & 0.00       & 0       & 0.00                                 &\textbf{10019}             &\textbf{105.46}\\
 fishing derby      & -91.7     & \textbf{71}        & 31.3     &75.60   & 44.85      & 83.93      & 45.2    & 84.14                                & 53.24            & 89.08 \\
 freeway            & 0         & \textbf{38}        & 34       &89.47   & 0          & 0.00       & 0       & 0.00                                 & 3.46             & 9.11 \\
 frostbite          & 65.2      & \textbf{454830}    & 9590.5   &2.09    & 317.75     & 0.06       & 5083.5  & 1.10                                 & 1583            & 0.33        \\          
 gopher             & 257.6     & \textbf{355040}             & 70354.6  &19.76   & 66782.3    & 18.75      & 114820.7& 32.29                                & 188680          & 53.11 \\
 gravitar           & 173       & \textbf{162850}    & 1419.3   &0.77    & 359.5      & 0.11       & 1106.2  & 0.57                                 & 4311             & 2.54        \\
 hero               & 1027      & \textbf{1000000}            & 55887.4  &5.49    & 33730.55   & 3.27       & 31628.7 & 3.06                        & 24236            & 2.32 \\
 ice hockey         & -11.2     & \textbf{36}                 & 1.1      &26.06   & 3.48       & 31.10      & 17.4    & 60.59                                & 1.56             & 27.03 \\
 jamesbond          & 29        & \textbf{45550}              & 19809    &43.45   & 601.5      & 1.26       & 37999.8 & 83.41                                & 12468           & 27.33 \\
 kangaroo           & 52        & \textbf{1424600}            & 14637.5  &1.02    & 1632       & 0.11       & 14308   & 1.00                        & 5399           & 0.38       \\
 krull              & 1598      & \textbf{104100}    & 8741.5   &6.97    & 8147.4     & 6.39       & 9387.5  & 7.60                                          & 64347            & 61.22              \\
 kung fu master     & 258.5     & \textbf{1000000}   & 52181    &5.19    & 43375.5    & 4.31       & 607443  & 60.73                                         & 124630.1            & 12.44        \\
 montezuma revenge  &0          & \textbf{1219200}   & 384      &0.03    & 0          & 0.00       & 0.3     & 0.00                                 & 2488.4           & 0.20       \\
 ms pacman          & 307.3     & \textbf{290090}    & 5380.4   &1.75    & 7342.32    & 2.43       & 6565.5  & 2.16                                 & 7579             & 2.51     \\
 name this game     & 2292.3    & 25220              & 13136    &47.30   & 21537.2    & 83.94      & 26219.5 & 104.36                               &\textbf{32098}             &\textbf{130.00}  \\
 phoenix            & 761.5     & \textbf{4014440}   & 108529   &2.69    & 210996.45  & 5.24       & 519304  & 12.92                                & 498590           & 12.40           \\
 pitfall            & -229.4    & \textbf{114000}    & 0        &0.20    & -1.66      & 0.20       & -0.6    & 0.20               & -17.8            & 0.19     \\
 pong               & -20.7     & \textbf{21}                 & 20.9     &99.76   & 20.98      & 99.95      & \textbf{21}      & \textbf{100.00}    & 20.39           & 98.54    \\
 private eye        & 24.9      & \textbf{101800}    & 4234     &4.14    & 98.5       & 0.07       & 96.3    & 0.07                                 & 134.1           & 0.11         \\
 qbert              & 163.9     & \textbf{2400000}   & 33817.5  &1.40    & 351200.12  & 14.63      & 21449.6 & 0.89                                 & 27371            & 1.13     \\
 riverraid          & 1338.5    & \textbf{1000000}   & 22920.8  &2.16    & 29608.05   & 2.83       & 40362.7 & 3.91                                 & 11182            & 0.99    \\
 road runner        & 11.5      & \textbf{2038100}   & 62041    &3.04    & 57121      & 2.80       & 45289   & 2.22                                 & 251360            & 12.33          \\
 robotank           & 2.2       & \textbf{76}                 & 61.4     &80.22   & 12.96      & 14.58      & 62.1    & 81.17                                & 10.44            & 11.17 \\
 seaquest           & 68.4      & \textbf{999999}             & 15898.9  &1.58    & 1753.2     & 0.17       & 2890.3  & 0.28                                 & 11862          & 1.18 \\
 skiing             & -17098    & \textbf{-3272}     & -12957.8 &29.95   & -10180.38  & 50.03      & -29968.4& -93.09                               & -12730            & 31.59       \\
 solaris            & 1236.3    & \textbf{111420}    & 3560.3   &2.11    & 2365       & 1.02       & 2273.5  & 0.94                                 & 2319           & 0.98      \\
 space invaders     & 148       & \textbf{621535 }   & 18789    &3.00    & 43595.78   & 6.99       & 51037.4 & 8.19                                 & 3031           & 0.46            \\
 star gunner        & 664       & 77400              & 127029   &164.67  & 200625     & 200.00     & 321528  & 200.00                               &\textbf{337150}            &\textbf{200.00}   \\
 surround           & -10       & 9.6                & \textbf{9.7}      &\textbf{100.51}  & 7.56       & 89.59      & 8.4     & 93.88              & -10              & 0.00 \\
 tennis             & -23.8     & \textbf{21}                 & 0        &53.13   & 0.55       & 54.35      & 12.2    & 80.36                                & -21.05              & 6.14 \\
 time pilot         & 3568      & 65300              & 12926    &15.16   & 48481.5    & 72.76      & \textbf{105316}  & \textbf{164.82}                               & 84341            & 130.84   \\
 tutankham          & 11.4      & \textbf{5384}      & 241      &4.27    & 292.11     & 5.22       & 278.9   & 4.98                                 & 381             & 6.88          \\
 up n down          & 533.4     & 82840              & 125755   &152.14  & 332546.75  & 200.00     & 345727  & 200.00                               &\textbf{416020}            &\textbf{200.00} \\
 venture            & 0         & \textbf{38900}     & 5.5      &0.01    & 0          & 0.00       & 0       & 0.00                                 & 0             & 0.00                 \\
 video pinball      & 0         & \textbf{89218328}  & 533936.5 &0.60    & 572898.27  & 0.64       & 511835  & 0.57                                 & 297920           & 0.33                  \\\
 wizard of wor      & 563.5     & \textbf{395300}    & 17862.5  &4.38    & 9157.5     & 2.18       & 29059.3 & 7.22                                 & 26008            & 6.45                \\
 yars revenge       & 3092.9    & \textbf{15000105}  & 102557   &0.66    & 84231.14   & 0.54       & 166292.3& 1.09                                 & 118730           & 0.77              \\
 zaxxon             & 32.5      & \textbf{83700}              & 22209.5  &26.51   & 32935.5    & 39.33      & 41118   & 49.11                                & 46070.8            & 55.03  \\
\hline
MEAN SABER(\%) &     0.00 & 100.00   &         & 28.39 &         & 29.45  &        & 36.78 &      &36.10\\
\hline
MEDIAN SABER(\%) & 0.00   & 100.00   &         & 4.92 &         & 4.31  &        & 8.08  &      &10.29  \\
\bottomrule
\end{tabular}
\end{center}
\clearpage

\begin{figure*}[t]
    \centering
    \vspace{-1.3cm}
    \includegraphics[width=1.0\linewidth]{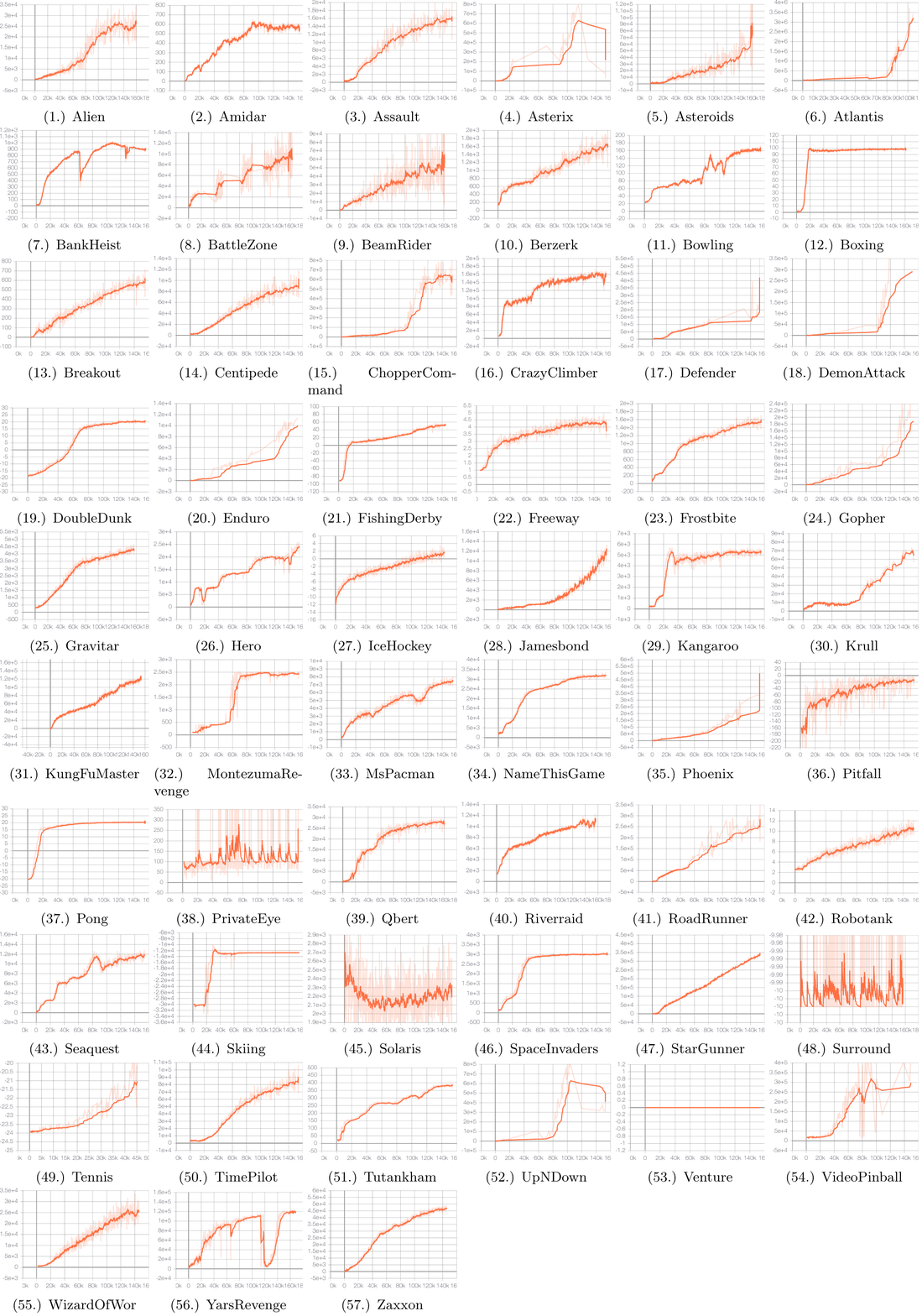}
\end{figure*}

\clearpage


\end{document}